\documentclass[11pt]{article}

\usepackage{epsfig,epsf,fancybox}
\usepackage{amsmath}
\usepackage{mathrsfs}
\usepackage{amssymb}
\usepackage{graphicx}
\usepackage{color}
\usepackage{multirow}
\usepackage{paralist}
\usepackage{verbatim}
\usepackage{galois}
\usepackage{algorithm}
\usepackage{algorithmic}
\usepackage{boxedminipage}
\usepackage{booktabs}
\usepackage{stmaryrd}
\usepackage{subfig}
\usepackage{wrapfig}
\usepackage[bottom]{footmisc}

\usepackage{hyperref}

\usepackage{nicefrac}
\usepackage[left=1in,top=1in,right=1in,bottom=1in,letterpaper]{geometry}

\newtheorem{theorem}{Theorem}[section]
\newtheorem{corollary}[theorem]{Corollary}
\newtheorem{lemma}[theorem]{Lemma}
\newtheorem{proposition}[theorem]{Proposition}
\newtheorem{definition}[theorem]{Definition}
\newtheorem{remark}[theorem]{Remark}

\numberwithin{equation}{section}

\providecommand{\keywords}[1]{\textbf{\textit{Keywords---}} #1}

\newcommand{\BE}{\mathbb{E}}

\newcommand{\T}{\mathrm{T}}

\newcommand{\diag}{\mathrm{diag}}
\newcommand{\grad}{\textnormal{grad}\,}

\newcommand{\Diag}{\mathrm{Diag}\,}

\newcommand{\retr}{\textnormal{Retr}}

\newcommand{\argmin}{\mathop{\rm argmin}}

\newcommand{\GCal}{\mathcal{G}}

\newcommand{\PScr}{\mathscr{P}}

\newcommand{\proj}{\textnormal{Proj}}

\newcommand{\St}{\textnormal{St}}
\newcommand{\Tg}{\textnormal{T}}

\newcommand{\br}{\mathbb{R}}

\newcommand{\be}{\begin{equation}}
\newcommand{\ee}{\end{equation}}
\newcommand{\ba}{\begin{array}}
\newcommand{\ea}{\end{array}}
\newcommand{\bad}{\begin{aligned}}
\newcommand{\ead}{\end{aligned}}

\newcommand{\WCal}{\mathcal{W}}

\newcommand{\MCal}{\mathcal{M}}
\newcommand{\M}{\mathcal{M}}

\newcommand{\PCal}{\mathcal{P}}

\newcommand{\NCal}{\mathcal{N}}

\newcommand{\Retr}{\mathrm{Retr}}
\newcommand{\onebf}{\mathbf{1}}
\newcommand{\WB}{\mathcal{W}\mathcal{B}}

\begin{document}

\title{Projection Robust Wasserstein Barycenters}

\author{Minhui Huang\thanks{Department of Electrical and Computer Engineering, University of California, Davis}
\and Shiqian Ma\thanks{Department of Mathematics, University of California, Davis}
\and Lifeng Lai\footnotemark[1]}
\date{\today}
\maketitle

\begin{abstract}
Collecting and aggregating information from several probability measures or histograms is a fundamental task in machine learning. One of the popular solution methods for this task is to compute the barycenter of the probability measures under the Wasserstein metric. However, approximating the Wasserstein barycenter is numerically challenging because of the curse of dimensionality. This paper proposes the projection robust Wasserstein barycenter (PRWB) that has the potential to mitigate the curse of dimensionality. Since PRWB is numerically very challenging to solve, we further propose a relaxed PRWB (RPRWB) model, which is more tractable. The RPRWB projects the probability measures onto a lower-dimensional subspace that maximizes the Wasserstein barycenter objective. The resulting problem is a max-min problem over the Stiefel manifold. By combining the iterative Bregman projection algorithm and Riemannian optimization, we propose two new algorithms for computing the RPRWB. The complexity of arithmetic operations of the proposed algorithms for obtaining an $\epsilon$-stationary solution is analyzed. We incorporate the RPRWB into a discrete distribution clustering algorithm, and the numerical results on real text datasets confirm that our RPRWB model helps improve the clustering performance significantly.

\end{abstract}

\keywords{Wasserstein Barycenter, Curse of Dimensionality, Riemannian Optimization}

\section{Introduction}

The Wasserstein barycenter (WB) problem is attracting a lot of interest recently due to its wide applications in statistics and machine learning, including but not limited to image processing \cite{rabin2011wasserstein}, multi-level clustering \cite{ho2017multilevel}, and text mining \cite{ye2017determining, ye2017fast}. The WB serves as a geodesic interpolation between two or more distributions. It aggregates the underlying geometric structures of the input distributions under the Wasserstein metric. Therefore, the WB model provides deep insight when collecting information from probability distributions.

However, computing the WB for a set of probability distributions is notoriously hard. The hardness comes from two aspects: the representation of the measure support and the curse of dimensionality. In many applications, the underlying distributions are unknown and we only have sampled data from these distributions. We wish to estimate the WB using the sampled data only. Therefore, the task reduces to compute WB from sampled discrete measures on a fixed number of support points. However, solving the free-support discrete WB is still very difficult \cite{borgwardt2019computational}. In this paper, we mainly consider the fixed-support WB problem. On the other hand, computing fixed-support WB can be challenging if the problem's dimension is high. Recent theoretical developments have revealed that the sample complexity of approximating Wasserstein distances grows exponentially in dimension \cite{dudley1969speed, weed2019sharp}. 
For the WB problem, \cite{altschuler2021wasserstein} has proved that computing WB is NP-hard since its runtime scales exponentially in the dimension. 
However, the sample complexity of original WB is still not well understood. Since the WB model minimizes a sum of Wasserstein distances, we conjecture that the WB problem would also have the issue of curse of dimensionality. To overcome this difficulty, we adopt a technique used in computing the Wasserstein distance \cite{paty2019subspace} to the WB problem. The resulting projection robust WB (PRWB) model is an inf-sup-inf problem, which is computationally intractable. We further propose a relaxation of PRWB that is computationally more tractable. The idea of the new technique is to project the sampled data to a common low dimensional subspace and compute the WB of the projected data as an approximation to the original WB. The resulting problem is a max-min problem with Stiefel manifold constraint, and we propose two algorithms that can find an $\epsilon$-stationary point of it efficiently.

\textbf{Related work:} Most existing works for fixed-support WB focus on designing efficient algorithms. \cite{cuturi2014fast} proposed to add an entropy regularizer and solve its dual problem that is smooth. This idea was further studied by \cite{benamou2015iterative} under the name of iterative Bregman projection (IBP) algorithm. The convergence behavior of IBP was studied in \cite{kroshnin2019complexity}. There exist some other algorithms for computing fixed-support WB, including the accelerated gradient descent method \cite{kroshnin2019complexity, lin2020fixed}, the stochastic gradient descent method \cite{claici2018stochastic}, the Bregman ADMM method \cite{ye2017fast} and the interior-point method \cite{ge2019interior}. On the other hand, people have proposed some efficient ways to mitigate the curse of dimensionality of the optimal transport (OT) problem \cite{niles2019estimation}. Specifically, \cite{bonneel2015sliced} proposed the sliced Wasserstein distance and applied it to the WB problem. The sliced OT projects the sampled data to a random line and reduces the problem to a one-dimensional OT, which can be solved very efficiently by sorting. This idea motivated the work of \cite{paty2019subspace,niles2019estimation} that suggest  projecting the data to a low dimensional subspace. This leads to the projection robust Wasserstein (PRW) distance, and algorithms for computing it include \cite{lin2020projection, huang2020riemannian}.

\textbf{Contributions:} Our main contributions are below.

(i) We propose a projection robust Wasserstein barycenter (PRWB) model. The PRWB model has the potential to mitigate the curse of dimensionality by projecting the probability measures onto a low dimensional subspace. Since PRWB is still numerically challenging to solve, we further propose a relaxation of PRWB (RPRWB) that is more tractable. Our numerical results indicate that RPRWB is more robust to noise compared with the WB.

(ii) We propose two algorithms: Riemannian block coordinate descent (RBCD) and Riemannian gradient ascent with IBP algorithm (RGA-IBP), to compute the RPRWB. The RGA-IBP incorporates the IBP algorithm to a Riemannian gradient ascent algorithm, and the RBCD is based on a reformulation of the max-min problem that is suitable for BCD type algorithms. The complexities of arithmetic operations of both algorithms for obtaining an $\epsilon$-stationary point are analyzed.

(iii) We conduct extensive numerical experiments to show the robustness and the practicality of the RPRWB model. We adopt RPRWB to the discrete distribution (D2) clustering algorithm, which we call the projected D2 clustering. We test this new algorithm on the real text datasets, and the numerical results show that the projected D2 clustering achieves better performance than the D2 clustering. 

\section{Optimal Transport and Wasserstein Barycenter}
In this section, we review some background in optimal transport and Wasserstein barycenter. 
Denote $\PScr (\br^d)$ as the set of Borel probability measures in $\br^d$ and $\PScr _2 (\br^d)$ as a subset of $\PScr(\br^d)$ whose elements have finite second moment. The 2-Wasserstein distance between probability measures $\mu, \nu \in \PScr _2 (\br^d)$ is defined as
\be \label{eq:Wdistance}\bad
\WCal(\mu, \nu)  :=  \left( \inf_{\pi \in \Pi(\mu, \nu)} \int \|x - y\|^2 d\pi(x, y) \right)^{1/2},
\ead\ee
where $\Pi(\mu, \nu)$ is the set of all joint distributions with marginals $\mu$ and $\nu$. We denote $\Delta^q = \{u \in \br_+^q | u^\top \onebf_q = 1\}$ as the probability simplex in $\br^q$. The WB of $m$ probability measures $\pmb{\mu} := \{\mu^l\}_{l \in [m]}$ is the solution of the following problem:
\be \label{eq:WBarycenter}\bad
  \inf_{\nu \in \PScr _2 (\br^d)}\  \WB(\pmb{\mu}, \pmb{\omega}) \ := \ \sum_{l=1}^m \omega^l \WCal^2 (\mu^l, \nu),
\ead\ee
where $\pmb{\omega} \in \Delta^m$ is a given weighting vector and $[m]:=\{1,\ldots,m\}$. We use $\proj_E$ to denote the orthogonal projector onto $E$ for any $E \in \GCal_k$, where the Grassmannian $ \GCal_k := \{E\in \br^d | \text{dim} (E) = k\}$ is the set of all $k$-dimensional subspaces of $\br^d$. For Wasserstein distance, \cite{paty2019subspace} proposed the projection robust Wasserstein distance as follows:
\be \label{eq:PRW}\bad
\PCal_k(\mu,\nu) := \sup_{E\in \GCal_k} \WCal(\proj_E \mu, \proj_E \nu).
\ead\ee
That is, the probability measures $\mu$ and $\nu$ are projected onto the $k$-dimensional subspace $E$, and the Wasserstein distance between the projected measures is computed as an approximation to the original Wasserstein distance. Moreover, to measure the worst case approximation, the subspace $E$ that maximizes this Wasserstein distance is sought. The study in \cite{niles2019estimation} shows that the projection robust Wasserstein distance is able to improve the sample complexity from $O(n^{-1/d})$ for Wasserstein distance to $O(n^{-1/k})$, where $n$ denotes the nubmer of sampled data. This is a significant improvement since usually $k\ll d$ for high dimensional OT. Therefore, the projection robust Wasserstein distance can mitigate the curse of dimensionality.

\section{Projection Robust Wasserstein Barycenter}

Our projection robust Wasserstein barycenter is motivated by the success of the projection robust Wasserstein distance and the sliced Wasserstein barycenter proposed in \cite{bonneel2015sliced}. By replacing the Wasserstein distance in \eqref{eq:WBarycenter} with the PRW distance \eqref{eq:PRW}, the fixed-support PRWB is defined as the solution of the following problem:
\be\label{PRWB-correct-1}
\inf_{\nu \in \PScr _2 (\br^d)}\  \sum_{l=1}^m \omega^l  \PCal_k^2(\mu^l,\nu).
\ee
Plugging \eqref{eq:PRW} into \eqref{PRWB-correct-1}, we have
 \be\label{PRWB-correct-2}\bad
&\inf_{\nu \in \PScr _2 (\br^d)}\  \sum_{l=1}^m \omega^l \sup_{E_\ell\in \GCal_k} \WCal^2(\proj_{E_\ell} \mu^l, \proj_{E_\ell} \nu) \\
=&\inf_{\nu \in \PScr _2 (\br^d)}\ \sum_{l=1}^m \omega^l \sup_{U_\ell\in \St(d,k)} \inf_{\pi^l\in\Pi(\mu^l,\nu)} \int \|U_\ell^\top (x^l - y) \|^2 d\pi^l(x^l, y).
\ead\ee
According to \cite{paty2019subspace}[Proposition 1], PRW is a well defined distance over $\PScr _2 (\br^d)$ and can be formulated as a {sup-inf} problem. Moreover, the support of the barycenter is fixed and our target barycenter $\nu$ lies on a probability simplex. Our PRWB formulation \eqref{PRWB-correct-2} is a inf-sup-inf problem over $m$ Stiefel manifolds. Solving \eqref{PRWB-correct-2} directly is extremely difficult, because of the complex inf-sup-inf structure and also the existence of $m$ Stiefel manifolds constraints. Therefore, we propose the following relaxation to PRWB \eqref{PRWB-correct-2} that is more computationally tractable:
\be \label{eq:PRWBsupinf}\bad
&\sup_{E\in \GCal_k} \inf_{\nu\in \PScr _2 (\br^d)}  
\sum_{l=1}^m \omega^l \WCal^2 (\proj_E \mu^l, \proj_E \nu) \\
=&\sup_{U\in \St(d,k)}\inf_{\nu \in \PScr _2 (\br^d)}\ \sum_{l=1}^m \omega^l  \inf_{\pi^l\in\Pi(\mu^l,\nu)} \int \|U^\top (x^l - y) \|^2 d\pi^l(x^l, y).
\ead
\ee
More specifically, we first use a common projector $\proj_E(\cdot)$ for all PRW distances, and then we switch the order of $\sup$ and the first $\inf$. The relaxed model \eqref{eq:PRWBsupinf} searches for a common low-dimensional subspace, the union of all subspaces of $m$ PRW distances, that maximizes the barycenter objective.
Roughly speaking, we solve an easier problem in a low-dimensional subspace to approximate the original WB problem.  
We call \eqref{eq:PRWBsupinf} the Relaxed PRWB (RPRWB) and focus on solving this relaxed version in the rest of the paper. We first study some properties of RPRWB. The following proposition shows the existence of the optimal subspace $E^*$.
\begin{proposition} \label{prop:Eexistence}
Given a probability measure set $\pmb{\mu}$, the support of the barycenter $\nu$, the weight vector $\pmb{\omega}$,  and $k \in [d]$, there exists an optimal $E^*$ for the problem \eqref{eq:PRWBsupinf}.
\end{proposition}
Notice that the target barycenter $\nu$ lies on a probability simplex. This combined with Proposition \ref{prop:Eexistence} indicates that the fixed-support RPRWB problem can be written as a max-min problem. Using $U \in \St(d,k)$ to denote an orthonormal basis of $E$, the RPRWB can be formulated as  
\be \label{eq:PRWB}\bad
\max_{U \in\St(d,k)} \min_{\pi^l \in \Pi(\mu^l, \nu)} \sum_{l=1}^m \omega^l  \int \|U^\top (x^l - y) \|^2 d\pi^l(x^l, y),
\ead\ee
where $\St(d,k)$ denotes the Stiefel manifold, $x^l$ is the support of $\mu^l$ and $y$ is the support of $\nu$.

\begin{remark}
We remark here that analyzing the sample complexity PRWB is highly nontrivial and the analysis in \cite{niles2019estimation} for PRW does not apply here. In fact, we are not aware of any results for the sample complexity of the empirical discrete WB problem. There are only some computational hardness results \cite{altschuler2021wasserstein} showing WB is NP-hard because of the ``curse of dimensionality''. 
Since the WB problem minimizes the sum of a set of Wasserstein distances, we conjecture that the ``curse of dimensionality'' should be inherited by WB. Deriving the sample complexity of WB and PRWB is an important future topic.
\end{remark}

In this paper, we consider solving WB for a set of discrete distributions. Specifically, we denote $X^l = [x^l_1; \cdots ; x^l_n] \in \br^{d \times n}$ as the support of each $\mu^l$ and write $\mu^l = \sum_{i =1}^n p^l_i \delta_{x^l_i}$, where $p^l \in \Delta^n$ and $\delta_x$ denotes the Dirac function at $x$. The support of the barycenter $\nu$ is given and denoted as $Y = \{y_1,\ldots,y_n\}\in \br^{d \times n}.$ Therefore, the barycenter can be written as $\nu = \sum_{i =1}^n q_j \delta_{y_j}$ with $q \in \Delta^n.$ Denote $\pmb{\pi} = \{\pi^l\}_{l \in [m]}$. Throughout this paper, we denote $\MCal = \St(d,k)$. Computing the fixed-support RPRWB is equivalent to solving
\be\label{eq:PRWBdiscrete}\bad
&\max_{U \in \MCal}   \min_{q \in \Delta^n } \sum_{l = 1}^{m} \omega^l  \WCal^2( \proj_E\mu^l, \proj_E \nu) = \max_{U \in \MCal} \min_{\pi\in\Pi(\pmb{p})} f(\pmb{\pi}, U),
\ead
\ee
where $ f(\pmb{\pi}, U) :=  \sum_{l = 1}^{m}\omega^l  \sum_{i, j = 1}^n  \pi_{i,j}^l \| U^\top(x^l_i - y_j)\|^2$, and $\Pi(\pmb{p}) = \{\pmb{\pi} \mid \pi^l \in \br^{n\times n}_+, \  \pi^l \onebf = p^l,  (\pi^l)^\top \onebf = (\pi^{l+1})^\top \onebf, \  l \in [m]\}$.

\section{The Riemannian Gradient Ascent and Riemannian BCD Algorithms}

In this section, we propose two algorithms for solving \eqref{eq:PRWBdiscrete}: RGA-IBP and RBCD. We can show that both algorithms find an $\epsilon$-stationy point of
\eqref{eq:PRWBdiscrete} defined as follows.
\begin{definition}\label{def:primalsta}
We call $(\pmb{\hat{\pi}}, \hat{U})\in\Pi(\pmb{p})\times\MCal$ an $\epsilon$-stationary point of the fixed-support RPRWB problem \eqref{eq:PRWBdiscrete}, if the following two inequalities hold:
\begin{align}
\|\emph{grad}_Uf(\pmb{\hat{\pi}}, \hat{U})\|_F \  & \leq \epsilon, \label{def:primalsta-eq-1} \\
f(\pmb{\hat{\pi}}, \hat{U}) -  f(\pmb{\pi}^*(\hat{U}), \hat{U})  & \leq \epsilon, \label{def:primalsta-eq-2}
\end{align}
where $\emph{grad}_Uf(\pmb{\hat{\pi}}, \hat{U})$ is the Riemannian gradient w.r.t. $U$, $\pmb{\pi}^*(\hat{U})$ is the optimal solution of the inner minimization problem of  \eqref{eq:PRWBdiscrete} when fixing $U$ as $\hat{U}.$
The corresponding $\epsilon$-approximate barycenter $q \in \Delta^n$ can be computed as $q = (\hat{\pi}^l)^\top \onebf, \forall l \in [m].$
\end{definition}
Before we present the algorithms, we define some useful notation first.
\begin{definition} (Cost and Correlation Matrices)
Given the support vectors $\{X^l\}_{l \in [m]}$ and $Y$, the cost matrices, denoted as $\{C^l\}_{l \in [m]}$, are defined as $C^l_{i,j} = \|x^l_i - y_j\|^2, \forall l \in [m]$. The correlation matrix, denoted as $V_{\pmb{\pi}}$, is defined as $V_{\pmb{\pi}} = \sum_{l = 1}^{m}\omega^l \sum_{i, j = 1}^n \pi_{i,j}^l (x^l_i - y_j )(x^l_i - y_j )^\top \in \br^{d\times d}.$
\end{definition} 

\subsection{The Riemannian Gradient Ascent with IBP Iterations}
The RGA-IBP algorithm is a natural extension of the RGAS algorithm (Riemannian gradient ascent with Sinkhorn's iteration) that was proposed by \cite{lin2020projection} for computing the projection robust Wasserstein distance. Here we extend it to solve the RPRWB problem \eqref{eq:PRWBdiscrete}.
The RGA-IBP algorithm solves the following problem, which is obtained by adding an entropy regularization to \eqref{eq:PRWBdiscrete}.
\be\label{eq:PRWBdiscretereg}\bad
& \max_{U \in \MCal} \min_{\pmb{\pi}\in\Pi(\pmb{p})}f_\eta(\pmb{\pi}, U) := \sum_{l = 1}^{m}\omega^l \left(\sum_{i, j = 1}^n  \pi_{i,j}^l \| U^\top(x^l_i - y_j)\|^2  - \eta H(\pi^l)\right),
\ead\ee
where $H(\pi) := - \sum_{i,j = 1}^n (\pi_{i,j} \log \pi_{i,j} - \pi_{i,j})$ is the entropy regularizer, and $\eta>0$ is a weighting parameter. Define
\be\label{eq:fU}
f_\eta(U) :=  \min_{\pmb{\pi}\in\Pi{(\pmb{p})}, \pi^l\in\Delta^{n^2}, \forall l\in [m]}f_\eta(\pmb{\pi}, U).
\ee
Note that in the minimization problem \eqref{eq:fU} we have added $m$ redundant constraints $\pi^l\in\Delta^{n^2}, \forall l\in [m]$, comparing to the minimization problem in \eqref{eq:PRWBdiscretereg}. The reason for adding these reduandant constraints will be clear later when we analyze the convergence of the algorithms. We know that \eqref{eq:PRWBdiscretereg} is equivalent to the following Riemannian optimization problem with smooth objective $f_\eta(U)$:
\be\label{max-fU}
\max_{U \in \MCal} f_\eta(U).
\ee
Problem \eqref{max-fU} can be naturally solved by a Riemannian gradient ascent algorithm whose $t$-th iteration is:
\[U_{t+1} := \Retr_{U_t}(\tau\grad f_\eta(U_t))\]
where $\Retr$ denotes the retraction operation, $\grad f_\eta$ denotes the Riemannian gradient of $f_\eta$, and $\tau>0$ is a step size. Moreover, it is easy to verify that
\be\label{eq:gradfU}\bad
\text{grad} f_\eta(U) =  \proj_{\Tg_U\MCal} (2V_{\pmb{\pi}_\eta^*(U)} U),
\ead\ee
where $\Tg_U\MCal$ denotes the tangent space of $\MCal$ at $U$, and $\pmb{\pi}_\eta^*(U)$ is the optimal solution of \eqref{eq:fU} that can be found by the IBP algorithm (see details in Algorithm \ref{alg:IBP}). 
The RGA-IBP algorithm is detailed in Algorithm \ref{alg:RGA-IBP}, where 
the IBP solver solves \eqref{eq:fU} up to an accuracy ${\epsilon}$ (see Algorithm \ref{alg:IBP} in the supplementary material).

\begin{algorithm}[H]
\caption{The RGA-IBP Algorithm}
\label{alg:RGA-IBP}
\begin{algorithmic}[1]
\STATE  \textbf{Input:} $\{\mu^l = (X^l, p^l)\}_{l\in [m]}$, $\{Y\}$, accuracy tolerance $\epsilon>0$. Set parameters $\tau$, $\eta$, and $\rho$ as in \eqref{lem:dualmain-param}.
\STATE \textbf{Initialization:} $U_0\in\St(d, k)$.
\FOR{$t = 0, 1, 2, \ldots,$}
\STATE $\pmb{\pi}_{t+1}$ = IBPsolver($\pmb{\mu}$, $Y$, $U_t$, $\eta$, ${\epsilon}$);
\STATE $\xi_{t+1} = \proj_{\Tg_{U_t}\MCal} (2V_{\pmb{\pi}_{t+1}}U_t)$;
\STATE $U_{t+1} = \retr_{U_t}(\tau \xi_{t+1})$;
\IF{$ \|\xi_{t+1} \|_F \le \epsilon$}
\STATE break;
\ENDIF
\ENDFOR
\STATE Output: $\hat{U} = U_t$, $\hat{\pmb{\pi}} = \pmb{\pi}_{t+1}$.
\end{algorithmic}
\end{algorithm}

\subsection{The Riemannian Block Coordinate Descent Algorithm}

Notice that the RGA-IBP requires to solve an optimization problem \eqref{eq:fU} in each iteration using an iterative solver. This can be quite expensive in practice. In this section, we propose the RBCD algorithm that can alleviate this computational burden. The RBCD algorithm presented here can be regarded as an extension of the algorithm recently proposed in \cite{huang2020riemannian} for computing the projection robust Wasserstein distance.

First, note that the optimization problem in \eqref{eq:fU} is convex and we have the following result about its dual. 
{
\begin{lemma}\label{lem:fU-dual}
The dual problem of \eqref{eq:fU} is equivalent to the following problem:
\be\label{eq:fU-dual}
\bad
\max_{\pmb{u, v} \in \br^{m\times n}, \sum_{l = 1}^{m} \omega^lv^l = 0} -\sum_{l = 1}^{m}  \omega^l \left\{ \log\left( \sum_{i, j = 1}^n  \zeta^l_{ij} \right)  - \langle u^l , p^l\rangle\right \},
\ead\ee
where $\zeta^l_{ij} =  [\zeta(u^l, v^l, U)]_{ij}$ is given by:
\be\label{eq:zeta}\bad
 \zeta(u^l, v^l, U)_{ij} =  \exp \left( - \frac{\| U^\top (x^l_i - y_j)\|^2}{\eta} + u^l_i + v^l_j \right),
\ead\ee
and the corresponding primal optimal solution $\pi^l_{ij} = [\pi(u^l, v^l, U)]_{ij}$ is
\be\label{eq:pisolution}\bad
\pi(u^l, v^l, U)_{ij} =  \frac{\zeta(u^l, v^l, U)_{i,j}}{\|\zeta(u^l, v^l, U) \|_1}.
\ead\ee
As a result, we know that \eqref{eq:PRWBdiscretereg} is equivalent to:
\be \label{eq:dualformu}\bad
\min_{\substack{ \pmb{u, v} \in \br^{m\times n},\  U \in \MCal, \\ \sum_{l = 1}^{m} \omega^lv^l = 0}}& g(\pmb{u}, \pmb{v}, U) := \sum_{l = 1}^{m}  \omega^l \left\{ \log\left( \sum_{i, j = 1}^n  \zeta^l_{i,j} \right) - \langle u^l , p^l\rangle\right \}.
\ead\ee
\end{lemma}}
Note that \eqref{eq:dualformu} has three block variables and it is suitable for block coordinate descent method. Our RBCD for solving \eqref{eq:dualformu} updates the iterates as follows:
{
\begin{align}
\pmb{u}_{t+1} & \in \argmin_{\pmb{u}} g(\pmb{u}, \pmb{v}_t, U_t) \label{eq:u-problem}\\
\pmb{v}_{t+1} & = \argmin_{\pmb{v}:\sum_{l = 1}^{m} \omega^lv^l = 0} g(\pmb{u}_{t+1}, \pmb{v}, U_t)\label{eq:v-problem}\\
U_{t+1} & := \retr_{U_t} ( - \tau \text{grad}_{U} g(\pmb{u}_{t+1}, \pmb{v}_{t+1}, U_t))\label{eq:RGDstep}.
\end{align}
It is easy to verify that \eqref{eq:u-problem} has multiple optimal solutions, and one of them is given below as a closed-form solution:
\be\label{eq:IBPstep-1}
u^l_{t+1} = u^l_t + \log \frac{p^l}{ \zeta (u^l_t, v^l_t, U_t)\onebf}, \forall l \in [m].
\ee
Problem \eqref{eq:v-problem} admit a unique solution that is given by 
\begin{align}\label{eq:IBPstep-2}
v^l_{t+1}  = v^l_t + \log \frac{q_{t+1}}{ q^l_t}, \forall l \in [m]
\end{align}
where we denote $q^l_t =( \zeta(u^l_{t+1}, v^l_t, U_t))^\top \onebf $ and $ q_{t+1} = \exp(\sum_{l = 1}^m \omega^l \log q^l_t )$. Notice that \eqref{eq:IBPstep-1}-\eqref{eq:IBPstep-2} renormalize the sum of rows and columns of each $\pi^l$ to be $p^l$ and $q_{t+1}$, which yields $\langle q^l_t, \onebf \rangle = 1, \forall l \in [m]$. Moreover, the update \eqref{eq:RGDstep} requires to compute $\text{grad}_{U} g$, and from \eqref{eq:pisolution} and \eqref{eq:dualformu} we know that
\be\label{eq:gradUg}\bad
\text{grad}_U g(\pmb{u}, \pmb{v}, U) = \proj_{\Tg_U\M} (-\frac{2}{\eta} V_{\pmb{\pi}(\pmb{u}, \pmb{v}, U) }U).
\ead\ee}

By combining \eqref{eq:u-problem}-\eqref{eq:gradUg}, we can summarize the details of the RBCD in Algorithm \ref{alg:RBCD-IBP}, in which we have adopted the following notation for the simplicity of presentation:
\[\bar{c} := \max_l\|C^l\|_\infty, \ \underline{\omega} = \min_l \omega^l.\] 
Note that in Algorithm \ref{alg:RBCD-IBP} we adopted a rounding procedure for the output. This is because that $\pi$ computed according to \eqref{eq:pisolution} does not necessarily lie in the constraint set $\Pi(\pmb{p})$. The rounding procedure proposed in  \cite{altschuler2017near} and outlined in Algorithm \ref{alg:round} can help round the solution to set $\Pi(\pmb{p})$. Note that this rounding procedure is also adopted in the IBP algorithm and thus in the RGA-IBP algorithm.

\begin{algorithm}[t]
\caption{The RBCD Algorithm}
\label{alg:RBCD-IBP}
\begin{algorithmic}[1]
\STATE \textbf{Input:} $\{\mu^l = (X^l, p^l)\}_{l\in [m]}$, $\{Y\}$, accuracy tolerance $\epsilon>0$. Set parameters $\tau$, $\eta$ and $\rho$ as in \eqref{thm:RBCD-param}.
\STATE \textbf{Initialization:} $U_0\in\St(d, k)$, $\pmb{u}_0, \pmb{v}_0\in\br^{m \times n}$,
\FOR{$t = 0, 1, 2, \ldots,$}
\STATE Compute $\pmb{u_{t+1}}, \pmb{v_{t+1}}$ by \eqref{eq:IBPstep-1}-\eqref{eq:IBPstep-2};
\STATE Compute $U_{t+1}$ by \eqref{eq:RGDstep};
\IF{$ \sum_{l = 1}^{m}  \omega^l \| q^l_t - \bar{q}_t \|_1 \le \underline{w}^{3/2}\epsilon/(12\bar{c}),$ and $ \eta\| \text{grad}_U g(\pmb{u}_{t+1}, \pmb{v}_{t+1}, U_t) \|_F \le \epsilon/3$}
\STATE break;
\ENDIF
\ENDFOR
\STATE \textbf{Output:}   $q = \bar{q}_t = \sum_{l = 1}^m \omega^l q^l_t$, $\hat{U} = U_t$, and $\hat{\pi}^l= Round(\pi^l({u}^l_{t+1}, {v}^l_t, {U}_t), p^l, q), \forall l \in [m]$.
\end{algorithmic}
\end{algorithm}

\begin{algorithm}[ht]
\caption{$Round(\pi, p, q)$ }
\label{alg:round}
\begin{algorithmic}[1]
\STATE \textbf{Input:} $\pi \in \br^{n\times n}$, $p \in\br^{n}$, $ q \in\br^{n}$.
\STATE $X = \Diag(x)$ with $x_i = \frac{p_i}{[\pi\onebf]_i} \wedge 1$
\STATE $\pi' = X\pi$
\STATE $Y = \Diag(y)$ with $y_j = \frac{q_j}{[(\pi')^\top\onebf]_j} \wedge 1$
\STATE $\pi'' = \pi'Y$
\STATE $err_p = p - \pi''\onebf, err_q = q - (\pi'')^\top\onebf$
\STATE \textbf{Output:} $\pi'' + err_perr_q^\top / \|err_p\|_1$.
\end{algorithmic}
 \end{algorithm}

\section{Convergence Analysis} 
In this section, we give the complexities of both the iteration number and the arithmetic operations for both RGA-IBP and RBCD for obtaining an $\epsilon$-stationary point of \eqref{eq:PRWBdiscrete} as defined in Definition \ref{def:primalsta}. The proofs are provided in the supplementary materials.

The next theorem and corollary are for RGA-IBP algorithm.
\begin{theorem} \label{thm:RGAmain}
Choose parameters
\be\label{lem:dualmain-param}
\tau = \frac{1}{8L_2\bar{c} + 2\rho L_1^2 }, \  \eta=\frac{\epsilon}{4\log(n)+2}, \
\rho = 2\bar{c}+ \frac{4\bar{c}^2}{\eta}.
\ee
The Algorithm \ref{alg:RGA-IBP} returns an $\epsilon$-stationary point defined in Definition \ref{def:primalsta} in
\be\label{lem:dualmain-T}
T =O(\log(n)^2L_1^2\bar{c}^4 k\epsilon^{-4})
\ee
iterations. 
\end{theorem}

\begin{corollary}\label{cor:RGAmain}
The per iteration arithmetic operations complexity of Algorithm \ref{alg:RGA-IBP} is
$
O(mn^2dk + mn^2\bar{c}^6\log(n)^2\epsilon^{-6}).
$
Therefore, the total arithmetic operations complexity of Algorithm \ref{alg:RGA-IBP} is
\be\label{eq:rgacomp}\bad
& O(mn^2dk + mn^2\bar{c}^6\log(n)^2\epsilon^{-6})\cdot O(\log(n)^2L_1^2\bar{c}^4k \epsilon^{-4}) \\
= & \log(n)^2L_1^2\bar{c}^4k O(mn^2dk\epsilon^{-4} + mn^2\bar{c}^6\log(n)^2\epsilon^{-10}).
\ead \ee
\end{corollary}

The next theorem and corollary are for RBCD algorithm.
\begin{theorem} \label{thm:RBCDmain}
Choose parameters
\be\label{thm:RBCD-param}
\tau = \frac{1}{4  L_2  \bar{c} /\eta + \rho L_1^2}, \ \eta=\frac{\epsilon}{4\log(n)+2},
\rho = \frac{2\bar{c}}{\eta} + \frac{4\bar{c}^2}{\eta^2}.
\ee
The Algorithm \ref{alg:RBCD-IBP} returns an $\epsilon$-stationary point defined in Definition \ref{def:primalsta} in
\be\label{thm:dualmain-T}
T =O\left(L_1^2\bar{c}^5\log(n)\underline{\omega}^{-3}\epsilon^{-3} \right)
\ee
iterations. 
\end{theorem}

\begin{corollary} \label{cor:RBCDmain}
The per iteration arithmetic operations complexity of Algorithm \ref{alg:RBCD-IBP} is
$
O(mn^2d).
$
Therefore, the total arithmetic operations complexity of Algorithm \ref{alg:RBCD-IBP} is
\be\label{eq:rbcdcomp}\bad
O\left(mn^2\log(n)dL_1^2\bar{c}^5k\underline{\omega}^{-3}\epsilon^{-3} \right).
\ead\ee
\end{corollary}

\begin{remark}
Comparing \eqref{eq:rgacomp} with \eqref{eq:rbcdcomp}, we see that RBCD has a better complexity dependence on $\epsilon$ and $\bar{c}$. However, RBCD has an extra term $\underline{\omega} \le \frac{1}{m}$. Therefore, RGA-IBP has a better theoretical complexity when $m$ is large.
\end{remark}

\section{Numerical Experiments} \label{sec:experiment}
In this section, we conduct numerical experiments on both synthetic datasets and real datasets to evaluate the proposed RPRWB model \eqref{eq:PRWBdiscrete}. For the synthetic dataset, we consider solving RPRWB for a set of Gaussian distributions, which has closed-form solutions \cite{alvarez2015note}. We compare the convergence rate of WB and RPRWB to the ground truth for the sampled discrete distributions, as well as the robustness against noise for the RPRWB model.  For real datasets, we incorporate the RPRWB model to the discrete distribution (D2) clustering algorithm \cite{ye2017fast} and test it on text datasets. All experiments are conducted on a Linux server with a 32-core Intel Xeon CPU (E5-2667, v4, 3.20GHz per core).

\subsection{Synthetic Dataset}
\paragraph{Multi-variable Gaussian Distributions:} It is well-known that the Wasserstein barycenter of a set of multi-variable Gaussian distributions $\{\mu^l\}_{l \in [m]}$ with $\mu^l = \NCal(a^l, \Sigma^l)$, where $a^l$ is the mean and $\Sigma^l$ is the covariance matrix, has a closed-form formula. Specifically, we have the following theorem.

\begin{theorem}[\cite{alvarez2015note}] Let $\mu^1, . . . , \mu^m$ be Gaussian distributions with respective means $a^1, \dots , a^m$ and covariance matrices $\Sigma^1 , ..., \Sigma^m$. The barycenter of $\mu^1,..., \mu^m$ with weights $\omega^1,..., \omega^m$ is the Gaussian distribution with mean $\bar{a} = \sum_{l = 1}^m \omega^l a^l$ and covariance matrix $\Sigma$ defined as the only positive definite matrix satisfying the equation
\be \label{eq:optcondition} \bad
S = \sum_{l = 1}^m \omega^l \left( S^{1/2} \Sigma^l S^{1/2} \right)^{1/2}.
\ead\ee
\end{theorem}
In this subsection, we compute the WB and RPRWB of a given set of zero-mean multi-variable Gaussian distributions $\{\mu^l\}_{l \in [m]}$, $\mu^l = \NCal(0, \Sigma^l).$ We set $\omega^l = \frac{1}{m}, \forall l \in [m]$ in all experiments.  

\paragraph{The dependence of RPRWB on $k$.} We first explore the dependence of the objective function value on $k.$ For each $\mu^l$, we sample an empirical measure $\mu^l_n$. Specifically, we sample $n$ points according to the Gaussian distribution $\mu^l = \NCal(0, \Sigma^l)$ to form the support matrix $X^l \in \br^{d\times n}$ and set $p^l_i = \frac{1}{n}, i \in [n]$. We set each of the covariance matrices $\Sigma^l$ to be a SPD matrix with rank $k^*$. Therefore, $X^l $ lies in a $k^*$-dimensional subspace and the barycenter of $\{\mu^l_n\}_{l \in [m]}$ should be in a $(m \times k^*)$-dimensional subspace. The support of the barycenter $Y\in \br^{d\times n}$ is obtained by applying k-means clustering on $X = [X^1; \cdots ; X^m] \in \br^{d \times mn}.$ We set parameters as $d = 100, m = 3, n = 10$. We further set the step size $\tau = 0.0005$ for both RBCD and RGA-IBP algorithms and $\eta = 0.5 \cdot \text{mid}(\{C^l\}_{l \in [m]})$, where $\text{mid}(\{C^l\}_{l \in [m]})$ is the median of the entries of $\{C^l\}_{l \in [m]}.$
\begin{figure}[h]
\centering
\includegraphics[width=0.7\textwidth]{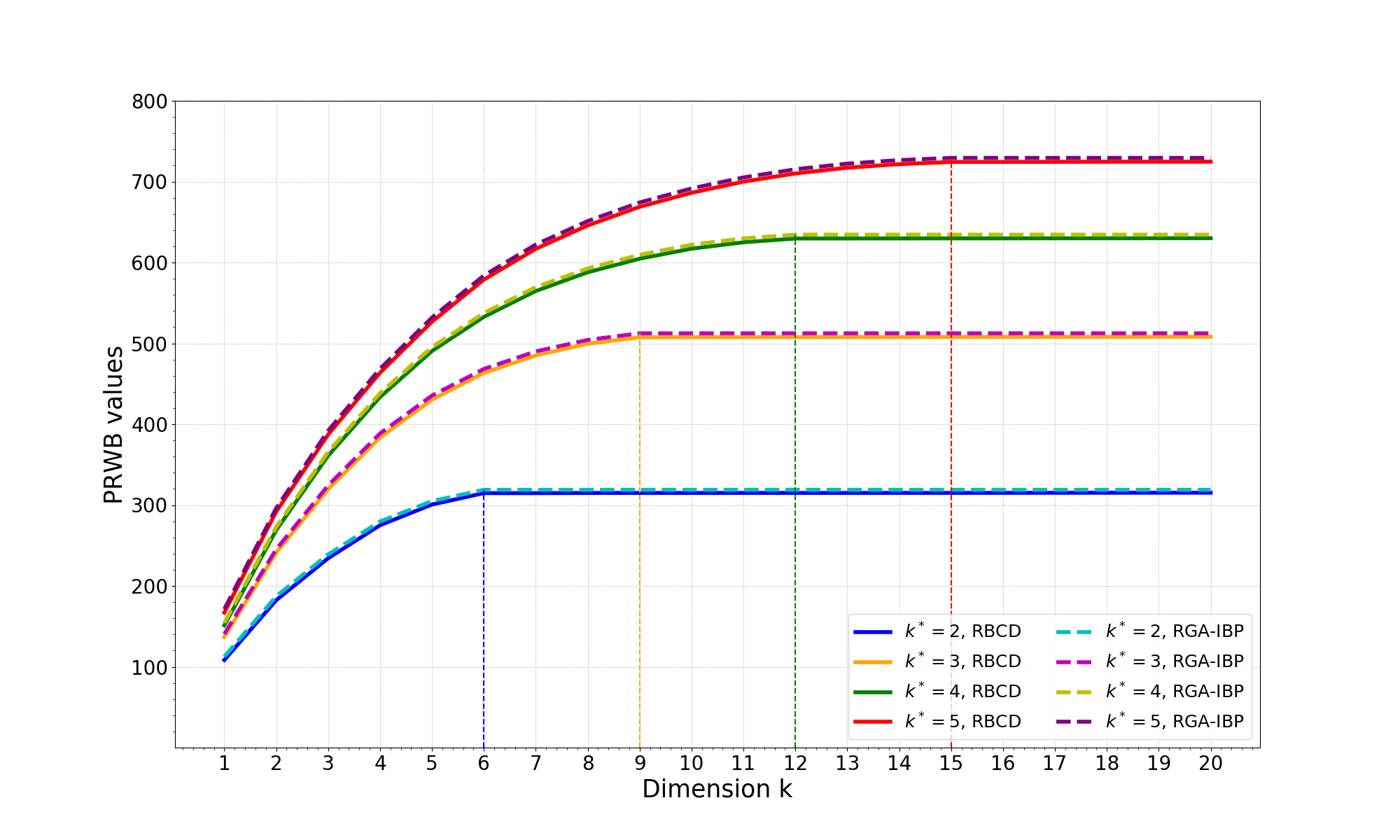}
\vspace*{-.5em}\caption{RPRWB function value versus projection dimension $k$. We run Algorithms \ref{alg:RBCD-IBP} and \ref{alg:RGA-IBP} on different $k$ and averaging over 100 runs,}
\label{fig:val_on_k}
\end{figure}

We run both RBCD and RGA-IBP for solving \eqref{eq:PRWBdiscrete} with different $k^*$ and $k$, and report the results in Figure \ref{fig:val_on_k}. From Figure \ref{fig:val_on_k} we see that the RPRWB values computed by the two algorithms are almost the same. We also notice that the RPRWB value increases when $k < m \times k^*$ and remains as a constant when $k \ge m \times k^*$, which verifies the fact that the barycenter of $\{\mu^l_n\}_{l \in [m]}$ lies in a $(m \times k^*)$-dimensional subspace.

\paragraph{Robustness Against Noise.} We further conduct experiments on comparing the robustness of WB and RPRWB against noise. Specifically, we add Gaussian noise $\sigma \mathcal{N}(0, I)$, where $\sigma$ is the noise level, to the discrete support $\{X^l\}_{l\in[m]}$. We compare the relative error of the objective function value for WB and RPRWB under different noise level $\sigma$. The relative error for WB and RPRWB is defined as
$$\text{Relative Error} = \frac{OBJ(\{\mu^l_n\}_\sigma) - OBJ(\{\mu_n^l\}_0))}{OBJ(\{\mu_n^l\}_0)},$$
where $\{\mu_n^l\}_\sigma$ denotes the distributions after adding noise $\sigma \mathcal{N}(0, I)$ and $OBJ$ denotes the objective function of WB (the discrete version of \eqref{eq:WBarycenter}) or RPRWB \eqref{eq:PRWBdiscrete}. We set parameters as $d = 100, m = 3, n = 10, \sigma \in [0.01, 0.1, 1, 2, 4, 7, 10]$. We choose the step size $\tau = 0.001$ when $\sigma < 7$ and $\tau = 0.0005$ otherwise for both RBCD and RGA-IBP algorithms and $\eta = 0.5 \cdot\text{mid}(\{C^l\}_{l \in [m]})$. The results are shown in Figure \ref{fig:noise_level}, which shows that the proposed RPRWB model is more robust to noise compared to the WB.  

\begin{figure}[h]
\centering
\includegraphics[width=0.5\textwidth]{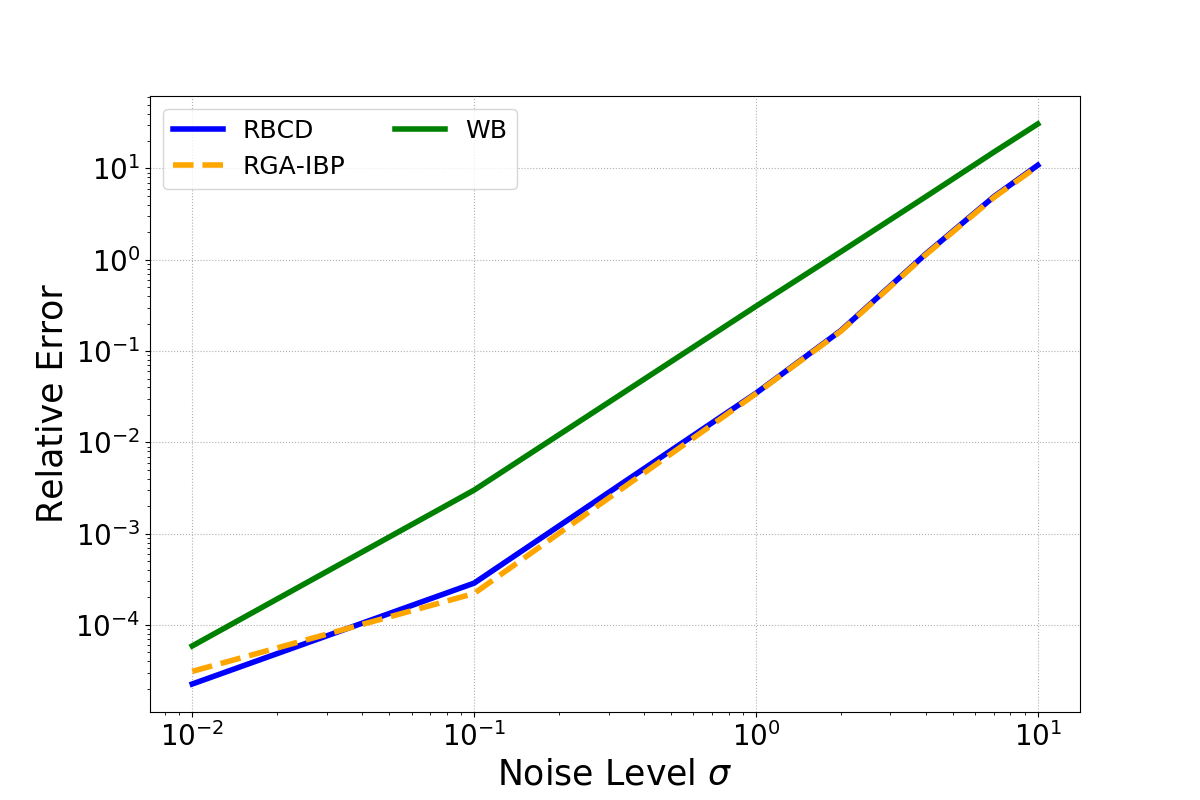}
\vspace*{-.5em}\caption{Relative error of the WB and RPRWB function value on different noise level $\sigma$. The results are averaged on 100 runs.}
\label{fig:noise_level}\vspace*{-1.5em}
\end{figure}

\paragraph{Convergence rate to the ground truth.} We further consider approximating the Wasserstein barycenter for a set of continuous distributions by sampling data. Note that \cite{niles2019estimation} proved that for a so-called spiked transport model, the mean projection robust Wasserstein distance between the sampled empirical distributions is $O(n^{-1/k})$, which improves the corresponding complexity of $O(n^{-1/d})$. We conjecture that similar results hold for WB and RPRWB and give some numerical evidence in this section. We set $d = 10, m = 2, k = 2$. The covariance matrices $\Sigma^1, \Sigma^2 \in \br^{d \times d}$ are diagonal matrices with $\Sigma^1(1, 1) = 10.1$, $\Sigma^2(2, 2) = 10.1$ and the rest of diagonal elements are all $0.1$. In this case, a 2-dimensional subspace catches most of the information about the barycenter. We then sample $n$ points as the support for each of $\mu^l$. To have a better estimation, the probability for $x_i^l$ is computed according to the Gaussian PDF:
$$Prob(x_i^l) = \frac{1}{(2\pi)^{d/2} \text{det}(\Sigma^l)^{1/2}} e^{-\frac{1}{2}(x_i^l)^T (\Sigma^l)^{-1} (x_i^l)}.$$
We sampled the support of the barycenter $Y$ according to a uniform distribution over $[-2, 2]^d$. The barycenter Mean Estimation Error is defined as
$$MEE = | OBJ(\{\mu^l\}) -  OBJ(\{\mu_n^l\})|,$$
where the ground truth objective function $OBJ(\{\mu^l\})$ is calculated by solving \eqref{eq:optcondition} and $OBJ(\{\mu_n^l\})$ is the sampled barycenter objective function value of the WB or RPRWB model. We set the step size $\tau = 0.05$ for both RBCD and RGA-IBP algorithms and $\eta = 0.5 \cdot\text{mid}(\{C^l\}_{l \in [m]})$ and select $n \in \{20, 50, 100, 250, 500, 1000 \}.$ The results are shown in Figure \ref{fig:est_err}, which shows that the proposed RPRWB model converges to the ground truth much faster than the WB.

\begin{figure}[h]
\centering
\includegraphics[width=0.5\textwidth]{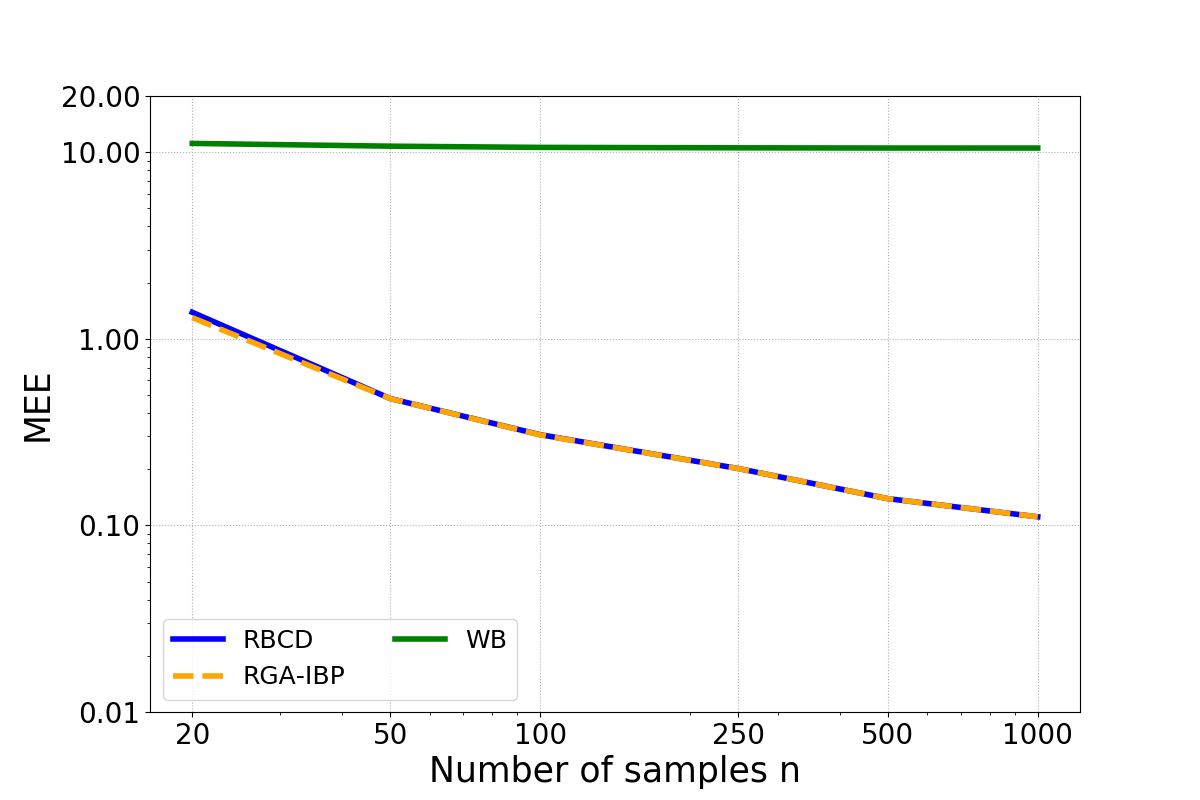}
\vspace*{-.5em}\caption{Mean Estimation Error (MEE) on different $n$. The results are averaged on 500 runs.}
\label{fig:est_err}
\end{figure}

\paragraph{Computational time comparison.} We compare the mean computational time of the WB solved by the IBP algorithm \cite{benamou2015iterative} and the proposed RPRWB solved by RBCD and RGA-IBP. We set $d = 100, m = 3, k = 2$ and select $n \in \{20, 50, 100, 250, 500, 1000 \}.$ We generate the support matrices $X^l \in \br^{d\times n}$ from $\mu^l = \mathcal{N}(0, \Sigma^l)$ by empirical sampling. The support of the barycenter $Y\in \br^{d\times n}$ is obtained by k-means clustering. We further set the step size $\tau = 0.01$ for both RBCD and RGA-IBP algorithms and $ \eta = 0.5 \cdot\text{mid}(\{C^l\}_{l \in [m]})$. We stop the RBCD algorithm when $\eta \|\text{grad}_U g(\pmb{u}_{t+1}, \pmb{v}_{t+1}, U_t) \|_F \le \epsilon,  \frac{1}{m}\sum_{l = 1}^{m} \| q^l_t - \bar{q}_t \|_1 \le \epsilon, $ and the RGA-IBP algorithm when $\|\text{grad}_U f(\xi_{t+1})\|_F \le \epsilon$, and we set $\epsilon = 10^{-4}.$ The results are shown in Figure \ref{fig:time}, which shows that RBCD always runs faster than RGA-IBP. Note that the  IBP for solving WB runs much faster than the other two algorithms, and this is because the latter two solve a more difficult problem.  

\begin{figure}[h]
\centering
\includegraphics[width=0.5\textwidth]{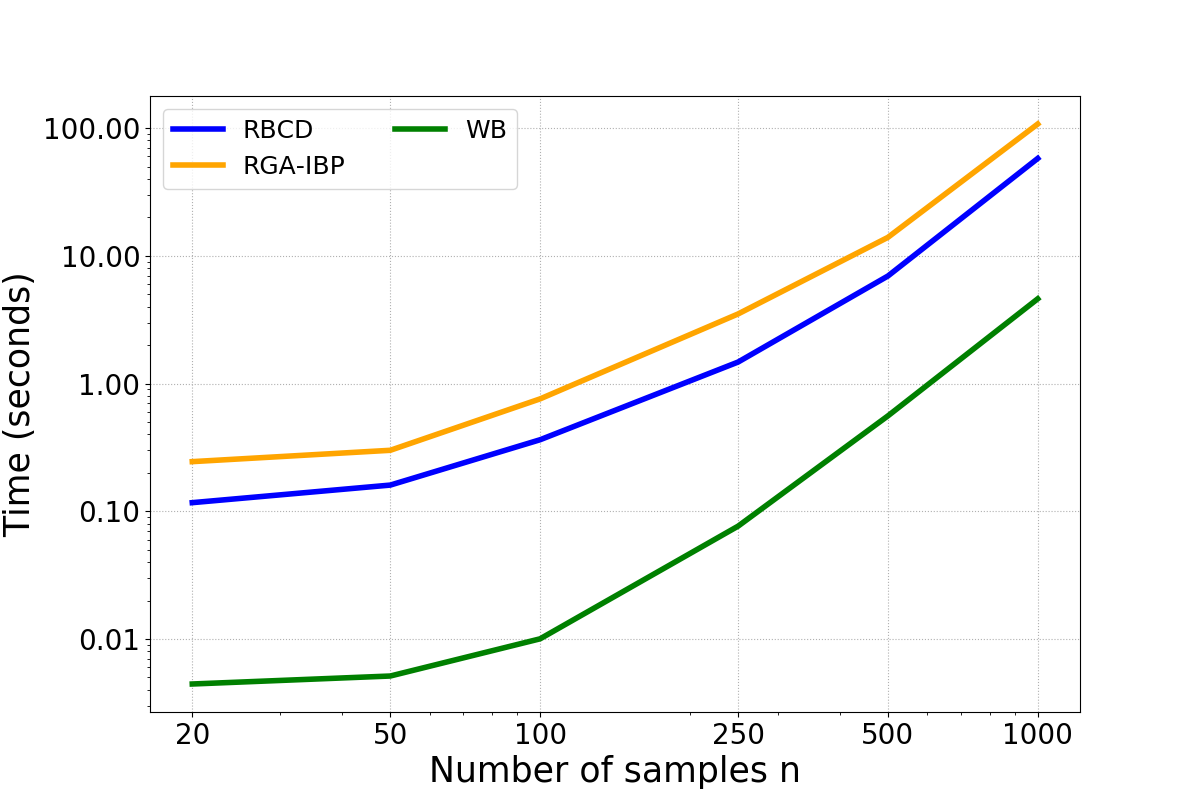}
\vspace*{-.5em}\caption{Computational time of the WB model solved by the IBP algorithm and the RPRWB model solved by RBCD and RGA-IBP algorithms on different $n$. The results are averaged on 100 runs.}
\label{fig:time}
\end{figure}

\begin{figure*}[h]
\centering
\includegraphics[width=0.45\textwidth]{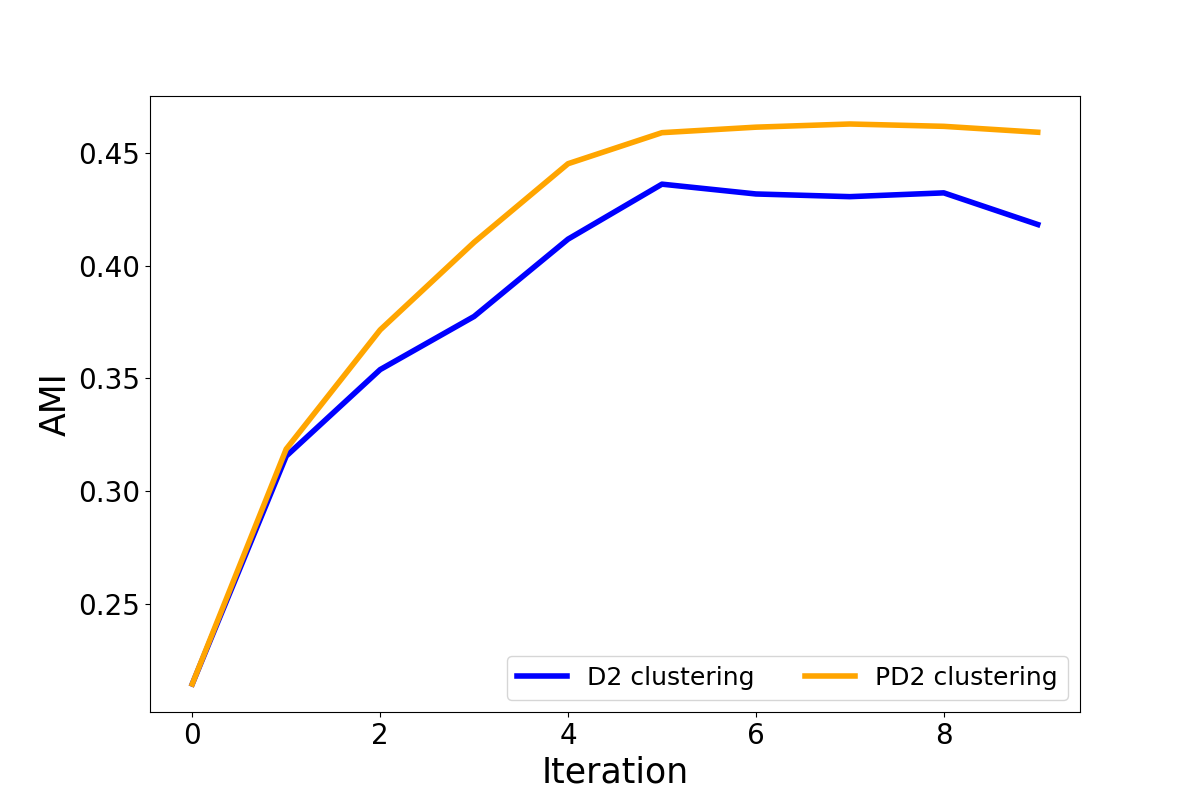}
\includegraphics[width=0.45\textwidth]{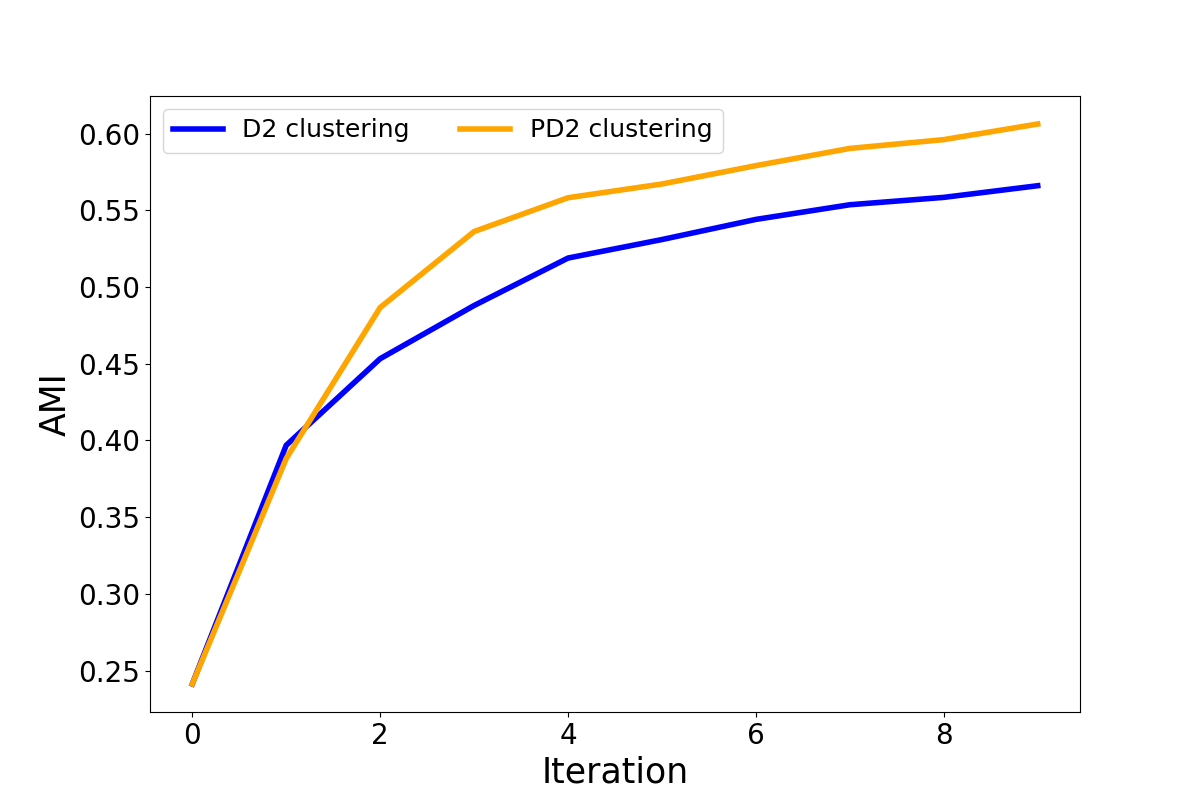}
\includegraphics[width=0.45\textwidth]{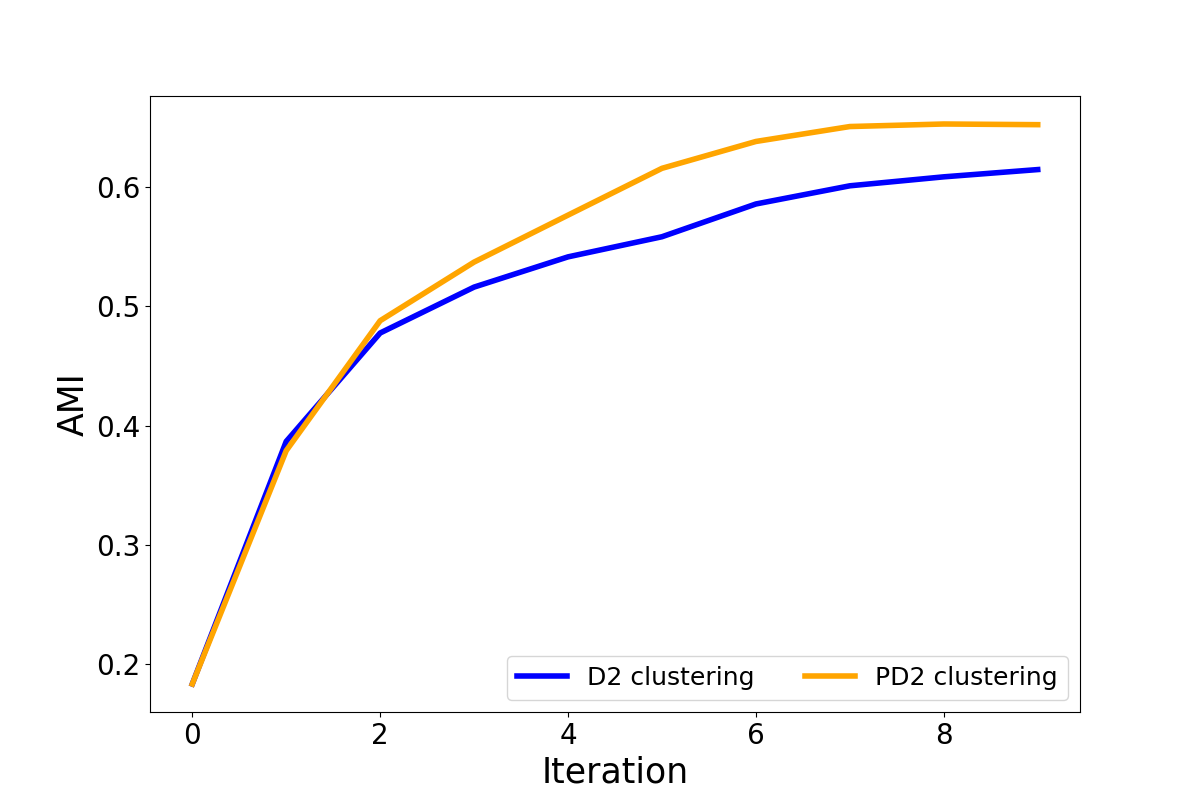}
\vspace*{-.5em}\caption{AMI scores for each iteration. \textbf{Top left:} the ``Reuters Subset'' dataset, \textbf{Top right:} the ``BBCsport Abstract'' dataset, \textbf{Bottom:} the ``BBCnews Abstract''  dataset. The results are averaged on 5 runs.}
\label{fig:ami}
\end{figure*}

\subsection{Real Dataset: text data}

We consider the discrete distributions (D2) clustering model proposed in \cite{ye2017fast}, which requires to solve the free-support discrete Wasserstein barycenter model:
\be\bad \label{eq:wbfree}
   & \min_{\nu }  \frac{1}{m}\sum_{l = 1}^{m}  \WCal( \mu^l,  \nu) \\
=  &\min_{\pmb{\pi}\in\Pi(\pmb{p}), Y \in \br^{d\times n}}\frac{1}{m} \sum_{l = 1}^{m} \sum_{i, j = 1}^n  \pi_{i,j}^l \| x^l_i - y_j\|^2,
\ead\ee
Note that there are two block variables: $\pmb{\pi}$ and $Y$. \cite{ye2017fast} proposed to solve \eqref{eq:wbfree} using an alternating mimization algorithm. That is, one alternatingly minimizes the objective function \eqref{eq:wbfree} with respect to one variable and with the other one fixed. This procedure is repeated until no progress can be made. 
When $Y$ is fixed, problem \eqref{eq:wbfree} becomes a fixed-support WB problem. When $\Pi = [\pi^1; \cdots ; \pi^m]$ is fixed, we have the following closed-form solution for $Y$:
\be\label{eq:yupdate}\bad
y_i = \frac{1}{m q_i} \sum_{l = 1}^{m}  \sum_{j = 1}^{n}  \pi_{i,j}^l x_j^l,
\ead\ee
which can be written more compactly as $Y = \frac{1}{m} X\Pi^\top \diag(1/ q).$

Since we have numerically demonstrated that RPRWB might be a better model than WB, we propose to replace the WB problem in D2 clustering by RPRWB. We call the resulting algorithm projection robust D2 clustering (PD2 clustering). 
More details of the D2 and PD2 clusterings can be found in the supplementary material. We compare the performance of D2 and PD2 clusterings on three text datasets listed in Table \ref{tab:datasets}. The ``Reuters Subset" is a 5-class subset of the ``Reuters'' dataset \footnote{https://www.nltk.org/book/ch02.html}. The ``BBCnews Abstract'' and ``BBCsport Abstract''\footnote{http://mlg.ucd.ie/datasets/bbc.html} \cite{greene2006practical} are truncated versions of 2,225 and 737 posts. Each document retains only the title and the first sentence of the original post.

\paragraph{Preprocessing.} We follow the idea of treating each document as a bag of word-vectors. For all three datasets in Table \ref{tab:datasets}, we use the pre-trained word-vector dataset GloVe \cite{pennington2014glove} to transform a list of words to a measure over $\br^{300}.$ The weight of each word is the normalized frequency modified by the TF-IDF scheme. We use the GloVe 300d (word vectors $\in \br^{300}$) that was trained on 6 billion tokens and contains a 400,000 lower case vocabulary. Before we transform words into vectors, we lower the capital letters, remove all punctuations and stop words and lemmatize each document. Finally, we restrict the number of support points to $n$ by recursively merging the closest words. Specifically, when the number of different words in $\mu^l$ is larger than $n$, we solve the following discrete optimization problem:
\be\bad
\min_{i,j} p^l_i p^l_j \| x^l_i - x^l_j\|^2 / (p^l_i + p^l_j),
\ead\ee
and merge $p^l_i, p^l_j$ as $\bar{p}^l = p^l_i + p^l_j$, $\bar{x} = (p^l_i x^l_i - p^l_j x^l_j) / \bar{p}^l.$

\paragraph{Parameter setting and initialization.} In each iteration of PD2 clustering, we run the RBCD algorithm with the step size $\tau = 0.05,$ the regularization parameter $\eta = 1$. We choose $k=2$ for the ``BBCsport Abstract'' dataset and $k=3$ for the ``Reuters Subset" and the ``BBCnews Abstract'' datasets. The initial $K$ barycenters are chosen randomly from documents with more than $n$ different words and recursively merged so the number of support points remains $n$.

\begin{table}[t]
\caption{Text Datasets. $N$ is the number of data, $n$ is the number of samples, and $K$ is the number of clusters.}
\label{tab:datasets}
\vskip 0.15in
\begin{center}
\begin{small}
\begin{tabular}{lcccr}
\toprule
 Dataset     & $N$  & $d$ & $n$  & $K$ \\
\midrule
Reuters Subset & 1209 & 300 & 16 &5     \\
BBCnews Abstract & 2225 & 300 & 16 &5     \\
BBCsport Abstract & 737 & 300 & 16  &5   \\
\bottomrule
\end{tabular}
\end{small}
\end{center}
\vskip -0.1in
\end{table}

\paragraph{The Adjusted Mutual Information.} To measure the performance of the clustering results, we use the Adjusted Mutual Information (AMI) \cite{vinh2010information}. Denote $P_V(i) = \lvert V_i\rvert /N$ as the probability of cluster $i$ in the partition $V$. The entropy  $H(V)$ is defined as  $H(V) = - \sum_{i = 1}^{S_V} P_V(i) \log P_V(i)$, where $S_V$ is the number of clusters in $V$. The mutual information between the two partitions $V_1, V_2$ is defined as $MI(V_1, V_2) = \sum_{i,j=1}^{S_{V_1}, S_{V_2}} P_{V_1,V_2}(i,j) \log \frac{P_{V_1,V_2}(i,j)}{P_{V_1}(i)P_{V_2}(j)},$ where $P_{V_1,V_2}(i,j) = \lvert V_{1,i} \cup V_{2,j}\rvert /N$. The AMI score between two partitions $V_1, V_2$ is computed by
\[
AMI(V_1,V_2) = \frac{MI(V_1,V_2) - \BE MI(V_1,V_2)}{(H(V_1) + H(V_2))/2 - \BE MI(V_1,V_2)}.
\]
The AMI score lies in the interval $[0, 1],$ and it remains unchanged when we permute the cluster labels. In our experiments, we present the AMI scores between the ground truth labels and the predicted labels.

\paragraph{Clustering results.} We run D2 and PD2 on the two real datasets in Table \ref{tab:datasets}. The final AMI score and the average number of iterations for different datasets are given in Tables \ref{tab:ami} and Table \ref{tab:avgiter} respectively. We apply k-means clustering on the raw TF-IDF vectors as a baseline. Each result is averaged over five runs with different initialization. We stop the D2 and PD2 algorithms when the labels for each cluster are stable. Comparing the AMI scores in Table \ref{tab:ami}, we see that the proposed PRWB model improves the performance of text clustering. One possible reason is that for many real high dimensional datasets, a low dimensional subspace catches most of the information. Notice that the D2 clustering AMI scores reported here are smaller than those in \cite{ye2017determining}. This is because the clustering performance highly depends on the barycenter initialization, and we are reporting the average AMIs with different initialization while \cite{ye2017determining} reported the best AMI they obtained. We further see that the average number of iterations of the PD2 algorithm is smaller. Moreover, we plot the AMI scores for the first ten iterations of the D2 and PD2 clustering algorithm in Figure \ref{fig:ami}. We see that the PD2 clustering algorithm gives better AMI scores than the D2 clustering algorithm, which shows the advantage of the proposed RPRWB model.

\begin{table}[htb]
\caption{AMI scores for clustering results.}
\label{tab:ami}
\vskip 0.15in
\begin{center}
\begin{small}
\begin{tabular}{lccc}
\toprule
 Dataset     & k-means  & D2 & PD2  \\
\midrule
Reuters Subset & 0.4627 & 0.4200 &  \textbf{0.4713}   \\
BBCnews Abstract& 0.3877 & 0.6095 &  \textbf{0.6557}   \\
BBCsport  Abstract& 0.4276 & 0.6510 & \textbf{0.6892}    \\
\bottomrule
\end{tabular}
\end{small}
\end{center}
\vskip -0.1in
\end{table}

\begin{table}[htb]
\caption{Average number of clustering iteration.}
\label{tab:avgiter}
\vskip 0.15in
\begin{center}
\begin{small}
\begin{tabular}{lccc}
\toprule
 Dataset     & D2 & PD2  \\
\midrule
Reuters Subset & 24.2 &  23.2  \\
BBCnews Abstract & 23.8 &  22.4   \\
BBCsport Abstract & 29.8 & 14.4    \\
\bottomrule
\end{tabular}
\end{small}
\end{center}
\vskip -0.1in
\end{table}

\section{Conclusion}
In this paper, we have proposed a novel WB model called the projection robust Wasserstein barycenter, which has the potential to mitigate the curse of dimensionality for the WB problem. To resolve the computational issue of the PRWB, we have proposed a relaxed PRWB model: RPRWB. We have proposed two algorithms, the RBCD algorithm and the RGA-IBP algorithm for solving the fixed-support RPRWB problem. We have analyzed the iteration complexity and complexity of arithmetic operations for both algorithms. Numerical results on synthetic datasets have demonstrated the robustness and the better sample complexity of the proposed RPRWB model comparing with the WB model. Moreover, we have incorporated the RPRWB model to the D2 clustering algorithm, and proposed the projection robust D2 clustering algorithm. Numerical results on real text datasets show that the PD2 clustering improves the performance of the D2 clustering. Future directions include deriving sample complexity for WB, PRWB and RPRWB.

\section*{Acknowledgements}
This work was supported in part by NSF HDR TRIPODS grant CCF-1934568, NSF grants CCF-1717943, CNS-1824553, CCF-1908258, ECCS-2000415, DMS-1953210 and CCF-2007797, and UC Davis CeDAR (Center for Data Science and Artificial Intelligence Research) Innovative Data Science Seed Funding Program.

\bibliography{PRWB}
\bibliographystyle{plain}

\appendix

\section{Preliminaries on Riemannian Optimization}
When considering optimization problems over the matrix manifold $\MCal$, the Riemannian Gradient Descent \cite{absil2009optimization} algorithm is widely used. A core ingredient of the Riemannian Optimization is the retraction operation:
\begin{definition}[\cite{absil2009optimization}]
\textbf{(Retraction)} Denote the tangent space of $\MCal$ at $U$ as $\Tg_U\MCal.$ A retraction on $\MCal$ is a smooth mapping $\retr(\cdot)$ from the tangent bundle $\Tg\MCal$ onto $\MCal$ satisfying the following two conditions:
\begin{itemize}
\item $\retr_U(0) = U$, $\forall U\in\M$, where $0$ denotes the zero element of $\Tg_U\MCal$;
\item For any $U \in \MCal$, it holds that
\[
\lim_{ \Tg_U\MCal \ni \xi \rightarrow 0} \|\retr_U(\xi) - (U + \xi)\|_F/\|\xi\|_F = 0.
\]
\end{itemize}
\end{definition}
In each step, the Riemannian Gradient Descent updates as
\be\bad
U_{t+1} = \retr(- \tau \text{grad} f(U_t)),
\ead\ee
where $\text{grad} f(U_t)$ is the Riemannian Gradient of $f(U)$ at $U_t$ defined as
\be\bad
\text{grad} f(U) = \proj_{\Tg_U\MCal} \nabla f(U).
\ead\ee
For the Stiefel manifold $\St(d, k) = \{U\in \br^{d \times k} | U^\top U = I_k\}$, the retraction has the following property.
\begin{proposition}[\cite{boumal2019global}]\label{pro:retr}
There exists constants $L_1, L_2 > 0$ such that for any $U \in \M$ and $\xi \in \Tg_U\M$, the following inequalities hold:
\begin{eqnarray*}
\|\retr_U(\xi) - U\|_F & \leq & L_1\|\xi\|_F, \\
\|\retr_U(\xi) - (U + \xi)\|_F & \leq & L_2\|\xi\|_F^2.
\end{eqnarray*}
\end{proposition}

\section{Proof of Proposition \ref{prop:Eexistence}}
\begin{proof}
Notice that the Grassmannian $\GCal_k$ is compact and the function $E \to \sum_{l=1}^m \omega^l \WCal^2 (\proj_E \mu^l, \proj_E \nu)$ is semi-continuous. These two facts lead to the desired result.
\end{proof}

\section{Proof of Lemma \ref{lem:fU-dual}}
\begin{proof}
We derive the dual problem of \eqref{eq:fU}. The Lagrangian function of \eqref{eq:fU} without considering the nonnegetivity constraints and the $m$ redundant constraints is:
\be\bad \label{eq:lagrangian}
L(\pmb{\pi}, U; \pmb{\alpha}, \pmb{\beta}) &= \sum_{l = 1}^{m}  \omega^l \left\{\sum_{i, j = 1}^n  \pi_{i,j}^l \| U^\top(x^l_i - y_j)\|^2- \eta H(\pi^l) \right\} + \sum_{l = 1}^{m} \langle \alpha^l , \pi^l \onebf - p^l \rangle + \sum_{l = 1}^{m} \langle \beta^l , (\pi^{l+1})^\top \onebf - (\pi^l)^\top \onebf  \rangle\\
&= \sum_{l = 1}^{m}  \omega^l \left\{\sum_{i, j = 1}^n  \pi_{i,j}^l \| U^\top(x^l_i - y_j)\|^2- \eta H(\pi^l) \right\} + \sum_{l = 1}^{m} \langle \alpha^l , \pi^l \onebf - p^l \rangle + \sum_{l = 1}^{m} \langle \beta^{l-1} - \beta^l , (\pi^l)^\top \onebf  \rangle
\ead\ee
where $\pmb{\alpha} = \{\alpha^l \in \br^{n}\}_{l \in [m]}, \pmb{\beta } = \{\beta^l\in \br^{n} \}_{l \in [0, ..., m]}$ are the Lagrange multipliers with $\beta^0 = \beta^m = 0.$ By changing the variables as $u^l = - \frac{\alpha^l}{\omega^l\eta} $, $v^l = -\frac{\beta^{l-1} - \beta^l }{\omega^l\eta}$, we can rewrite the Lagrangian function as:
\be\bad \label{eq:lagrangianuv}
L(\pmb{\pi}, U; \pmb{u}, \pmb{v}) = \eta \sum_{l = 1}^{m}  \omega^l \left\{ \sum_{i, j = 1}^n  \pi_{i,j}^l  \left( \frac{\| U^\top(x^l_i - y_j)\|^2}{\eta} + (\log \pi_{i,j} - 1 ) - u^l_i - v^l_j \right)   + \sum_i u^l_i p^l_i\right \}.
\ead\ee
Note that the change of the variables results in a further constraint $\sum_{l = 1}^{m} \omega^lv^l = 0$. We denote $\pmb{u} = \{u^l\}_{l \in [m]}, \pmb{v} = \{v^l\}_{l \in [m]} $. The dual problem of \eqref{eq:fU} is
\be\bad \label{eq:dualproblem}
&
\max_{\substack{ \pmb{u, v} \in \br^{m\times n}, \\ \sum_{l = 1}^{m} \omega^lv^l = 0}} \min_{\pi^l\in\Delta^{n^2}, \forall l\in [m]} L(\pmb{\pi}, U; \pmb{u}, \pmb{v}) =\\
& \max_{\substack{ \pmb{u, v} \in \br^{m\times n}, \\ \sum_{l = 1}^{m} \omega^lv^l = 0}} \eta \sum_{l = 1}^{m}  \omega^l \left\{ \min_{\pi^l\in\Delta^{n^2}, \forall l\in [m]} \sum_{i, j = 1}^n  \pi_{i,j}^l  \left( \frac{\| U^\top(x^l_i - y_j)\|^2}{\eta} + (\log \pi_{i,j} - 1 ) - u^l_i - v^l_j \right) + \langle u^l , p^l\rangle \right \}.
\ead\ee
{
For each minimization problem in \eqref{eq:dualproblem}, we know that it admits a closed-form solution given by:
\be\bad \label{eq:innerpi}
\pi^l_{i,j} = \frac{ \exp \left( - \frac{\| U^\top(x^l_i - y_j)\|^2}{\eta} + u^l_i + v^l_j \right)}{\sum_{i,j} \exp \left( - \frac{\| U^\top(x^l_i - y_j)\|^2}{\eta} + u^l_i + v^l_j \right)} , \ \forall l\in [m].
\ead\ee
Plugging \eqref{eq:innerpi} into \eqref{eq:dualproblem}, the dual problem becomes
\be\bad \label{eq:dualproblemfinal}
&\max_{\substack{ \pmb{u, v} \in \br^{m\times n}, \\ \sum_{l = 1}^{m} \omega^lv^l = 0}} \eta \sum_{l = 1}^{m}  \omega^l \left\{ - \log\left(\sum_{i, j = 1}^n \exp\left( - \frac{\| U^\top(x^l_i - y_j)\|^2}{\eta} + u^l_i + v^l_j \right)\right) +\langle u^l , p^l\rangle\right \}.
\ead\ee
This completes the proof.}
\end{proof}

\section{The Iterative Bregman Projection Algorithm}
For the completeness of Algorithm \ref{alg:RGA-IBP}, we include the IBP subroutine in this section, and present it as Algorithm \ref{alg:IBP}. When fixing $U$ in the dual formulation \eqref{eq:fU-dual}, the Iterative Bregman Projection algorithm updates $\pmb{u}, \pmb{v}$ by \eqref{eq:IBPstep-1} - \eqref{eq:IBPstep-2} in an alternative scheme. In Algorithm \ref{alg:IBP}, we define
\be\bad \label{eq:ibpeq}
&\pi(u^l_t, v^l_t, U)_{i,j} =  \exp \left( - \frac{\| U^\top(x^l_i - y_j)\|^2}{\eta} + u^l_{t,i} + v^l_{t, j} \right), \\
&q^l_t =( \pi(u^l_{t+1}, v^l_{t+1}, U))^\top \onebf, \\
& q_{t+1} = \exp\left(\sum_{l = 1}^m \omega^l \log q^l_t \right).
\ead\ee
The stopping criteria $ \sum_{l = 1}^{m}  \omega^l \| q^l_t - \bar{q}_t \|_1 \le \frac{\eta\epsilon^2}{200\bar{c}^3}$ guarantees the following inequality holds: 
\be\label{eq:ibpstopcondition}\bad
 \sum_{l = 1}^m \omega^l \|\pi^l_{t+1} - \pi_\eta^*(U_t)^l\|_1 \le 
 = \frac{\epsilon}{10 \bar{c}},
 \ead\ee
which will be proved in Section \ref{sec:rgathm}. 

\begin{algorithm}[t]
\caption{The Iterative Bregman Projection Algorithm --- IBPsolver($\pmb{\mu}$, $Y$, $U$, $\eta$, ${\epsilon}$)}
\label{alg:IBP}
\begin{algorithmic}[1]
\STATE \textbf{Input:} $\{\mu^l = (X^l, p^l)\}_{l\in [m]}$, $\{Y\}$, $U$, $\eta$,  accuracy tolerance $\epsilon>0$.
\STATE \textbf{Initialization:} $\pmb{u}_0, \pmb{v}_0\in\br^{m \times n}$.
\FOR{$t = 0, 1, 2, \ldots,$}
\STATE Compute $v^l_{t+1} = v^l_{t}  + \log(\frac{q_{t+1}}{q_t^l}), \quad \forall l \in [m]$;
\STATE Compute $u^l_{t+1} = u^l_{t}  + \log(\frac{p^l}{\pi (u^l_t, v^l_{t+1}, U) \onebf}), \quad \forall l \in [m]$;
\IF{$ \sum_{l = 1}^{m}  \omega^l \| q^l_t - \bar{q}_t \|_1 \le \frac{\eta\epsilon^2}{200\bar{c}^3}$}
\STATE break;
\ENDIF
\ENDFOR
\STATE \textbf{Output:}   $q = \bar{q}_t = \sum_{l = 1}^m \omega^l q^l_t$, and $\hat{\pi}^l= Round(\pi^l(\hat{u}^l, \hat{v}^l, \hat{U}), p^l, q), \forall l \in [m]$.
\end{algorithmic}
\end{algorithm}

\section{Proof of Theorem \ref{thm:RBCDmain} and Corollary \ref{cor:RBCDmain}}

{
Before we start the proof of Theorem \ref{thm:RBCDmain}, we first notice that \eqref{eq:IBPstep-1} renormalizes the row sum of each $\zeta^l$ to be $p^l$. Therefore, we have
\be\bad
\|\zeta^l(u^l_{t+1}, v^l_t, U_t)\|_1 = 1, \forall l \in [m], t > 0,
\ead\ee
which yields
\be\bad
\zeta^l(u^l_{t+1}, v^l_t, U_t) = \pi^l(u^l_{t+1}, v^l_t, U_t).
\ead\ee
Since \eqref{eq:IBPstep-2} renormalizes the column sum of each $\zeta^l$ to be $q_{t+1}$, we have
\be\label{eq:zeta-equal}\bad
\|\zeta^l(u^l_{t+1}, v^l_{t+1}, U_t)\|_1 = \|\zeta^k(u^k_{t+1}, v^k_{t+1}, U_t)\|_1, \quad \forall l, k \in [m].
\ead\ee
This fact combined with \eqref{eq:dualformu} and \eqref{eq:v-problem} lead to
\be\label{zeta-v-lessthan1}\bad
\|\zeta^l(u^l_{t+1}, v^l_{t+1}, U_t)\|_1 \le \|\zeta^l(u^l_{t+1}, v^l_{t}, U_t)\|_1 = 1.
\ead\ee
}

The proof of Theorem \ref{thm:RBCDmain} consists of two parts. We first prove that when the Algorithm \ref{alg:RBCD-IBP} stops, the output $(\hat{\pmb{\pi}}, \hat{U})$ is an $\epsilon$-stationary point defined in Definition \ref{def:primalsta}. Secondly, we give the iteration complexity of terminating the RBCD algorithm. Finally, we analyze the per iteration complexity of the Algorithm \ref{alg:RBCD-IBP}. Below, we list some  useful lemmas and theorems proved in the literature.

Lemma \ref{lem:roundcloseness} shows the bound of the difference between the input and the output of the Rounding procedure (Algorithm \ref{alg:round}).
\begin{lemma}\cite{altschuler2017near}[Lemma 7]\label{lem:roundcloseness}
Let $p, q \in \Delta^n$, $\pi\in\br_+^{n\times n}$, and $\hat{\pi}$ be the output of $Round(\pi, p,q)$. The following inequality holds:
\[
\|\hat{\pi} - \pi\|_1 \le 2(\|\pi\onebf - p\|_1 + \|\pi^\top \onebf - q\|_1).
\]
\end{lemma}

Theorem \ref{thm:variancebound} bounds the difference between the arithmetic mean and the geometric mean by the variance of the random variable.
\begin{theorem} \cite{aldaz2013monotonicity}[Theorem 2.4] \label{thm:variancebound}
For $n \ge 2$ and $l \in [m]$, let $X = (x_1, ... x_n)$ be such that $x_l \ge 0$, and let $\omega = (\omega^1, ..., \omega^l)$ satisfy $\omega^l >0 $ and $\sum_{l=1}^m \omega^l = 1.$ Then for all $s \in [1, \infty),$ we have
\[
 \mathbb{E}_\omega X - \Pi_\omega X \le \frac{1}{\min_l \omega^l} Var_\omega(X^{s/2})^{1/s},
\]
where $ \mathbb{E}_\omega X = \sum_{l = 1}^m \omega^l x_l$ is the arithmetic mean, $ \Pi_\omega X = \Pi_{l = 1}^m  x_l^{\omega^l}$ is the geometric mean and $ Var_\omega(X) = \sum_{l = 1}^m \omega^l (x_l -  \mathbb{E}_\omega X)^2$ is the variance.
\end{theorem}

The following lemma shows the relation between the primal and dual objective function.
{
\begin{lemma}\label{lem:primaldualrelation}
Denote $\pmb{\pi} = \{ \pi^l\}_{l = 1}^m$. By \eqref{eq:pisolution}, each $\pi^l$ can be written as
\[
\pi^l_{i,j} = \frac{\zeta^l_{i,j}}{\|\zeta^l\|_1},
\]
where
\[
\zeta^l_{i,j} = [\zeta(u_{t+1}^l, v_t^l, U_t)]_{i,j} =  \exp \left( - \frac{\| U_t^\top(x^l_i - y_j)\|^2}{\eta} + (u_{t+1}^l)_i + (v_t^l)_j \right).
\] 
 After $\pmb{u}$ is updated by \eqref{eq:IBPstep-1}, we have $\|\zeta^l(u_{t+1}^l,v_t^l,U_t)\|_1 = 1$, which yields the following equality:
\be\bad
f_\eta(\pmb{\pi}(u_{t+1},v_t,U_t), U_t) = -\eta g(\pmb{u}_{t+1}, \pmb{v}_t, U_t) +  \eta \sum_{l=1}^m \omega^l \langle v_t^l, q_t^l  \rangle - \eta,
\ead\ee
where $q_t^l = (\pi^l(u_{t+1}^l,v_t^l,U_t))^\top\onebf$, $f_\eta$ is defined in \eqref{eq:PRWBdiscretereg}, and $g$ is defined in \eqref{eq:dualformu}.  
\end{lemma}}
\begin{proof}
{Fix $U$ and denote $M^l \in \br^{n\times n}, l \in [m]$ with $M^l_{i,j} = \|U^\top(x_i^l - y_j)\|^2.$ Plugging \eqref{eq:innerpi} into $H(\pi^l)$ leads to
\be\bad
&f_\eta(\pmb{\pi}(u_{t+1},v_t,U_t), U_t)\\
= & \sum_{l=1}^m \omega^l ( \langle M^l,\pi^l(u_{t+1},v_t,U_t)\rangle -\eta H(\pi^l(u_{t+1},v_t,U_t) )) \\
=& \eta \sum_{l=1}^m \omega^l \left[\left( \sum_{i,j} \pi^l(u_{t+1},v_t,U_t)_{i,j} u^l_i + \pi^l(u_{t+1},v_t,U_t)_{i,j}v^l_j  -  \pi^l(u_{t+1},v_t,U_t)_{i,j} \right) - \log(\|\zeta^l(u_{t+1},v_t,U_t)\|_1)\right].
\ead\ee
The update rule \eqref{eq:IBPstep-1} yields $\zeta^l(u_{t+1},v_t,U_t) \onebf = p^l$. Therefore, we have $\|\pi^l(u_{t+1}^l,v_t^l,U_t)\|_1 = 1$ and
\be\label{eq:primal-dual-relation}\bad
& f_\eta(\pmb{\pi}, U) \\
= & - \eta \sum_{l=1}^m \omega^l \sum_{i,j} \left( \log(\|\zeta^l(u_{t+1},v_t,U_t)\|_1) -  \pi^l(u_{t+1},v_t,U_t)_{i,j} u^l_{ i} - {\pi}^l(u_{t+1},v_t,U_t)_{i,j}v^l_{j} +  \pi^l(u_{t+1},v_t,U_t)_{i,j}\right)\\
= & - \eta \sum_{l=1}^m \omega^l \left[\log(\|\zeta^l(u_{t+1},v_t,U_t)\|_1)  - \langle \pi^l(u_{t+1},v_t,U_t) \onebf ,  u^l\rangle -   \langle (\pi^l(u_{t+1},v_t,U_t))^\top\onebf,  v^l \rangle \right] - \eta\\
= & - \eta \sum_{l=1}^m \omega^l \left[ \log(\|\zeta^l(u_{t+1},v_t,U_t)\|_1)  - \langle p^l ,  u^l\rangle -   \langle v^l , q^l_t  \rangle \right] - \eta\\
= & -\eta g(\pmb{u}_{t+1}, \pmb{v}_t, U_t) +  \eta \sum_{l=1}^m \omega^l \langle v^l_t, q^l_t  \rangle - \eta,
\ead\ee
which completes the proof. } 
\end{proof}

The next lemma shows that when Algorithm \ref{alg:RBCD-IBP} terminates, it returns an $\epsilon$-stationary point of the problem \eqref{eq:PRWBdiscrete} as defined in Definition \ref{def:primalsta}.
\begin{lemma}\label{lem:primaldualstaequ}
Assume Algorithm \ref{alg:RBCD-IBP} terminates at iteration $T$. Then $(\hat{\pmb{\pi}},\hat{U})$ returned by Algorithm \ref{alg:RBCD-IBP}, i.e., $\hat{\pi}^l = Round(\pi(u^l_{T+1},v^l_T,U_T), p^l ,q)$ and $\hat{U} = U_T$, is an $\epsilon$-stationary point of the problem \eqref{eq:PRWBdiscrete} as defined in Definition \ref{def:primalsta}.
\end{lemma}

\begin{proof}
When Algorithm \ref{alg:RBCD-IBP} terminates at the $T$-th iteration, we have (note that $q = \bar{q}_T$):
\be\label{stop-kappa}
\sum_{l = 1}^{m}  \omega^l \| q^l_T - \bar{q}_T \|_1 \le \underline{\omega}^{3/2}\epsilon/(12\bar{c}),
\ee
and
\be\label{stop-U}
\eta\|\grad_U g(\pmb{u}_{T+1}, \pmb{v}_{T+1}, U_T )\|_F \leq \frac{\epsilon}{3}.
\ee
Fix $U_T$ and denote $\bar{\pi}^l  = \pi^l{(u^l_{T+1}, v^l_T, U_T)}$. By Lemma \ref{lem:primaldualrelation}, we have
\be\label{eq:primal-dual-relation}\bad
f_\eta(\bar{\pmb{\pi}}, U_T) = -\eta g(\pmb{u}_{T+1}, \pmb{v}_{T}, U_T) +  \eta \sum_{l=1}^m \omega^l \langle v_T^l, q^l_T  \rangle - \eta.
\ead\ee
{
Let $\pmb{\pi}_\eta^*(U_T)= \{(\pi^*_\eta(U_T))^l\}$ be the solution of \eqref{eq:fU} when fixing $U$ as $U_T.$ Denote $\pmb{u}_\eta^*, \pmb{v}_{\eta}^*$ as the corresponding optimal solution to the dual problem. We have the relation:
\be\bad \label{eq:innerpistar}
(\zeta_\eta^*(U))^l_{i,j} =  \exp \left( - \frac{\| U^\top(x^l_i - y_j)\|^2}{\eta} + (u_\eta^*)^l_i + (v_\eta^*)^l_j \right),
\ead\ee
and $\|(\zeta_\eta^*(U))^l\|_1 = 1$, which leads to $(\pi_\eta^*(U))^l = (\zeta_\eta^*(U))^l$. The optimal regularized Barycenter, denoted as $q_\eta^*$,  can be obtained by $q_\eta^* = [(\pi^*_\eta)^l]^\top \onebf= [(\pi^*_\eta)^{l+1}]^\top \onebf , \forall l \in \{1,..., m-1\} .$ Similar to Lemma \ref{lem:primaldualrelation}, we have
\be\label{eq:primal-dual-relation-optimal}\bad
f_\eta(\pmb{\pi}_\eta^*(U_T), U_T) & = -\eta g(\pmb{u}_{\eta}^*, \pmb{v}_{\eta}^*, U_T) +  \eta \sum_{l=1}^m \omega^l \langle (v_\eta^*)^l, q_\eta^* \rangle - \eta\\
&= -\eta g(\pmb{u}_{\eta}^*, \pmb{v}_{\eta}^*, U_T) - \eta,
\ead\ee
where the last equality uses the fact that $\sum_{l=1}^m \omega^l  (v_\eta^*)^l = 0$.
}
Let $b^l_T = \frac{\min_i [v^l _T]_i + \max_i [v^l _T]_i}{2}$ be a constant. Combining \eqref{eq:primal-dual-relation} and \eqref{eq:primal-dual-relation-optimal} yields
\be\label{eq:fetagrelation}\bad
f_\eta(\bar{\pmb{\pi}}, U_T) - f_\eta(\pmb{\pi}_\eta^*(U_T), U_T) & = -\eta ( g(\pmb{u}_{T+1}, \pmb{v}_{T}, U_T) - g(\pmb{u}_\eta^*, \pmb{v}_{\eta}^*, U_T) ) +  \eta \sum_{l=1}^m \omega^l \langle v^l_T, q^l_T \rangle\\
& \le \eta \sum_{l=1}^m \omega^l \langle v^l_T , q^l_T\rangle\\
& = \eta \sum_{l=1}^m \omega^l \langle v^l _T, q^l_T - \bar{q}_T \rangle\\
& = \eta \sum_{l=1}^m \omega^l \langle v^l _T - b^l_T \onebf, q^l_T - \bar{q}_T \rangle\\
& \le \eta \sum_{l=1}^m \omega^l  \| v^l_T - b^l_T\onebf\|_\infty\| q^l_T - \bar{q}_T \|_1,
\ead\ee
where the second equality follows from the constraint $\sum_{l=1}^m \omega^l  v^l = 0$ in \eqref{eq:v-problem}, the third equality is due to $\langle q^l_T, \onebf\rangle = 1,  \langle \bar{q}_T, \onebf\rangle = 1$, which follows from the optimality condition of \eqref{eq:u-problem}, and the last inequality is by H\"{o}lder's inequality. By \cite{kroshnin2019complexity}[Lemma 4], we can bound $\| v^l_T- b^l_T\onebf\|_\infty$ by $ \frac{\max_l \|M^l\|_\infty}{\eta} \le \frac{\max_l \|C^l\|_\infty}{\eta}$. Therefore, it holds that
\begin{align*}
f_\eta(\bar{\pmb{\pi}}, U_T) - f_\eta(\pmb{\pi}_\eta^*(U_T), U_T)  &=  \sum_{l=1}^m \omega^l ( \langle M^l,\bar{\pi}^l\rangle -\eta H(\bar{\pi}^l )) - \sum_{l=1}^m \omega^l ( \langle M^l, (\pi_\eta^*(U_T))^l \rangle -\eta H((\pi_\eta^*(U_T))^l))\\
&\le \bar{c} \sum_{l=1}^m \omega^l \| q^l_T - \bar{q}_T \|_1 \le \underline{\omega}^{3/2} \epsilon/12 \le \epsilon/12,
\end{align*}
where the second inequality is from \eqref{stop-kappa}.
Further denote $\pmb{\pi}^*(U_T))= \{(\pi^*(U_T))^l\}$ as the solution of the unregularized inner minimization problem of \eqref{eq:PRWBdiscrete} when fixing $U$ as $U_T.$  The above inequality implies
\be\label{eq:diffbarpistarpi}\bad
\sum_{l=1}^m \omega^l  \langle M^l,\bar{\pi}^l\rangle &\le \sum_{l=1}^m \omega^l \left( \langle M^l, (\pi_\eta^*(U_T))^l \rangle -\eta H((\pi_\eta^*(U_T))^l) +\eta H(\bar{\pi}^l) \right) +  \frac{ \epsilon}{12}\\
& \le \sum_{l=1}^m \omega^l \left( \langle M^l, (\pi^*(U_T))^l \rangle -\eta H((\pi^*(U_T))^l) +\eta H(\bar{\pi}^l) \right) +  \frac{ \epsilon}{12}\\
& \le \sum_{l=1}^m \omega^l  \langle M^l,(\pi^*(U_T))^l \rangle  +  \frac{ 7\epsilon}{12},
\ead\ee
where in the last inequality, we use the fact $0 \le H(\pi^l) \le 2\log n + 1$, and $\eta = \frac{\epsilon}{4\log n+ 2}.$ By Lemma \ref{lem:roundcloseness}, since $\hat{\pi}^l = Round(\bar{\pi}^l,p^l,q)$, Algorithm \ref{alg:round} outputs $\hat{\pi}^l$ satisfying
\be\label{diff-barpi-hatpi}
\|\hat{\pi}^l - \bar{\pi}^l\|_1 \le 2(\| \bar{\pi}^l \onebf - p^l\|_1 + \| (\bar{\pi}^l)^\top\onebf - q \|_1) = 2\|q^l_T - q\|_1, \quad \forall l \in [m],
\ee
where we note $q = \bar{q}_T$, and we have used the fact $\bar{\pi}^l \onebf - p^l=0$ that comes from the optimality condition of \eqref{eq:u-problem}. Equation
\eqref{diff-barpi-hatpi} further implies
\be \label{eq:roundpi}
\sum_{l=1}^m \omega^l \|\hat{\pi}^l - \bar{\pi}^l\|_1 \le 2 \sum_{l=1}^m \omega^l \|q^l_T - q\|_1.
\ee
Combining \eqref{eq:roundpi} with \eqref{eq:diffbarpistarpi} and applying H\"{o}lder's inequality yields
\be \label{eq:fhatp}\bad
\sum_{l=1}^m \omega^l  \langle M^l,\hat{\pi}^l\rangle  &\le  \sum_{l=1}^m \omega^l  \left( \langle M^l,\bar{\pi}^l\rangle + \|M^l\|_\infty \|\hat{\pi}^l -  \bar{\pi}^l \|_1\right)\\
&\le \sum_{l=1}^m \omega^l \langle M^l,\bar{\pi}^l\rangle  + (\max_l \|M^l\|_\infty) \sum_{l=1}^m \omega^l  \|\hat{\pi}^l -  \bar{\pi}^l \|_1\\
&\le \sum_{l=1}^m \omega^l \langle M^l,\bar{\pi}^l\rangle  + 2 (\max_l \|C^l\|_\infty) \sum_{l=1}^m \omega^l  \|q^l_t - q\|_1\\
&\le \sum_{l=1}^m \omega^l  \langle M^l, (\pi^*(U_T))^l \rangle  +  \frac{ 7\epsilon}{12} + \frac{\underline{\omega}^{3/2} \epsilon}{6}\\
&\le \sum_{l=1}^m \omega^l  \langle M^l, (\pi^*(U_T))^l \rangle  + \epsilon,
\ead\ee
where the fourth inequality follows from \eqref{stop-kappa}, and the last inequality holds since $\underline{\omega}^{3/2} < 1.$ Therefore, we have proved that $(\hat{\pi},\hat{U})$ satisfies \eqref{def:primalsta-eq-2} in Definition \ref{def:primalsta}.

{The rest of the proof is to prove that $(\hat{\pi},\hat{U})$ satisfies \eqref{def:primalsta-eq-1} in Definition \ref{def:primalsta}. That is, we need to bound $\|\text{grad}_U f(\hat{\pmb{\pi}}, U_T)\|_F$. For simplicity of notation, we further denote $\tilde{\pmb{\zeta}} = \pmb{\zeta}{(\pmb{u}_{T+1}, \pmb{v}_{T+1}, U_T)}$ and $\tilde{\pmb{\pi}} = \pmb{\pi}{(\pmb{u}_{T+1}, \pmb{v}_{T+1}, U_T)}$. By \eqref{eq:zeta-equal}, we have for any $l \in [m]$:
\[
\text{grad}_U f(\tilde{\pmb{\zeta}}, U_T) = 2V_{\tilde{\pmb{\zeta}}}U_T = 2\|\zeta^l\|_1 V_{\tilde{\pmb{\pi}}}U_T = \|\zeta^l\|_1 \text{grad}_U f(\tilde{\pmb{\pi}}, U_T).
\]
Since
\[
\text{grad}_U f(\tilde{\pmb{\pi}}, U_T) = -\eta \nabla_U g(\pmb{u}_{T+1}, \pmb{v}_{T+1}, U_T),
\]
we have
\[
\|\text{grad}_U f(\tilde{\pmb{\zeta}}, U_T)\|_F = \|\zeta^l\|_1 \|\text{grad}_U  f(\tilde{\pmb{\pi}}, U_T)\|_F \le \|\text{grad}_U f(\tilde{\pmb{\pi}}, U_T)\|_F \le \frac{\epsilon}{3},
\]
where the first inequality is due to \eqref{zeta-v-lessthan1} and the second inequality is due to \eqref{stop-U}. By triangle inequality, we have
\begin{align}\label{lem:dualtoprimal-proof-eq-1}
\|\text{grad}_U f(\hat{\pmb{\pi}}, U_T)\|_F &= \|\proj_{\T_{U_T}\M} (2 V_{\hat{\pmb{\pi}}}U_T)\|_F\\
& = \|\proj_{\T_{U_T}\M} (2 (V_{\hat{\pmb{\pi}}} -  V_{\bar{\pmb{\pi}}} + V_{\bar{\pmb{\pi}}} - V_{\tilde{\pmb{\zeta}}} + V_{\tilde{\pmb{\zeta}}})U_T)\|_F\nonumber\\
& \le  2\|(V_{\hat{\pmb{\pi}}} -  V_{\bar{\pmb{\pi}}})U_T\|_F + 2\|(V_{\bar{\pmb{\pi}}} - V_{\tilde{\pmb{\zeta}}})U_T\|_F + \|\proj_{\T_{U_T}\M} (2V_{\tilde{\pmb{\zeta}}}U_T)\|_F\nonumber\\
& \le  2\|V_{\hat{\pmb{\pi}}} -  V_{\bar{\pmb{\pi}}}\|_F + 2\|V_{\bar{\pmb{\pi}}} - V_{\tilde{\pmb{\zeta}}}\|_F + \|\text{grad}_U f(\tilde{\pmb{\zeta}}, U_T)\|_F\nonumber\\
& \le  2\|V_{\hat{\pmb{\pi}}} -  V_{\bar{\pmb{\pi}}}\|_F + 2\|V_{\bar{\pmb{\pi}}} - V_{\tilde{\pmb{\zeta}}}\|_F + \frac{\epsilon}{3} \nonumber.
\end{align}
In the following, we will bound $\|V_{\hat{\pmb{\pi}}} -  V_{\bar{\pmb{\pi}}}\|_F$ and $\|V_{\bar{\pmb{\pi}}} - V_{\tilde{\pmb{\zeta}}}\|_F$. 
Equations \eqref{eq:roundpi} and \eqref{stop-kappa} indicate that
\be\label{lem:dualtoprimal-proof-eq-2}
2\|V_{\hat{\pmb{\pi}}} -  V_{\bar{\pmb{\pi}}}\|_F \le 2 \bar{c} \sum_{l = 1}^m \omega^l  \|\hat{\pi}^l - \bar{\pi}^l\|_1 \le 4\bar{c} \sum_{l = 1}^m \omega^l \| q^l_T - q\|_1\le 4\bar{c}\cdot \frac{\underline{\omega}^{3/2} \epsilon}{12\bar{c}}  \le \frac{\epsilon}{3}.
\ee
We now bound $\|V_{\bar{\pmb{\pi}}} - V_{\tilde{\pmb{\zeta}}}\|_F$. Note that
\be\bad
   & \|\bar{\pi}^l -\tilde{\zeta}^l \|_1 \\
= & \sum_{ij} \left\lvert \exp{\left(-\frac{1}{\eta}\|(U_T)^\top(x^l_i - y_j)\|^2 + [u^l_{T+1}]_i + [v^l_{T}]_j\right)} -  \exp{\left(-\frac{1}{\eta}\|(U_T)^\top(x^l_i - y_j)\|^2 + [u^l_{T+1}]_i + [v^l_{T+1}]_j\right)} \right\rvert\\
\leq & \sum_{ij} \exp{\left(-\frac{1}{\eta}\|(U_T)^\top(x^l_i - y_j)\|^2 + [u^l_{T+1}]_i + [v^l_{T}]_j\right)} \lvert 1 -  \exp{([v^l_{T+1}]_j - [v^l_{T}]_j)} \rvert\\
= & \sum_j \sum_i \bar{\pi}^l_{i,j} \left\lvert 1 - \frac{[q_{T+1}]_j}{[q^l_T]_j} \right\rvert\\
= & \|q_{T+1} - q^l_T \|_1,
\ead\ee
where the second equality is from \eqref{eq:IBPstep-2}.
This leads to
\be\label{lem:dualtoprimal-proof-eq-3}
2\|V_{\bar{\pi}} - V_{\tilde{\zeta}}\|_F \le 2 \bar{c}\sum_{l = 1}^m \omega^l  \|\bar{\pi}^l -\tilde{\zeta}^l \|_1 \le 2\bar{c}\sum_{l = 1}^m \omega^l  \|q^l_T - q_{T+1}  \|_1 \le 2\bar{c} \left(  \| \bar{q}_T  - q_{T+1}  \|_1 + \sum_{l = 1}^m \omega^l  \|q^l_T - \bar{q}_T   \|_1  \right).
\ee
To complete the proof, it remains to bound $\| \bar{q}_T  - q_{T+1}\|_1.$ Notice that $ \bar{q}_T  = \sum_{l = 1}^m \omega^l q^l_T$ is the arithmetic mean and $q_{T+1} = \exp(\sum_{l = 1}^m \omega^l \log q^l_T)$ is the geometric mean. Setting $s = 2$ in Theorem \ref{thm:variancebound}, we have
\be\label{lem:dualtoprimal-proof-eq-4}\bad
 \| \bar{q}_T  - q_{T+1}  \|_1 & = \sum_{r=1}^n ( \bar{q}_{T,r} - q_{T+1,r})  \le\sum_{r=1}^n \frac{1}{\underline{\omega}} Var_{\omega}(q^l_{T,r})^{1/2}= \frac{1}{\underline{\omega}} \sum_{r=1}^n  \left[ \sum_{l=1}^m \omega^l ( q^l_{T,r} -  \bar{q}_{T,r})^2\right]^{1/2}\\
 & \le \frac{1}{\underline{\omega}} \sum_{r=1}^n  \left[ \frac{1}{\underline{\omega}^l}\sum_{l=1}^m [\omega^l ( q^l_{T,r} - \bar{q}_{T,r})]^2\right]^{1/2}  \le \frac{1}{{\underline{\omega}}^{3/2}} \sum_{r=1}^n  \left[ \left(\sum_{l=1}^m \omega^l \cdot \lvert  q^l_{T,r} - \bar{q}_{T,r} \rvert \right)^2   \right]^{1/2}\\
 & = \frac{1}{\underline{\omega}^{3/2}} \sum_{r=1}^n  \sum_{l=1}^m \omega^l \cdot \lvert  q^l_{T,r} - \bar{q}_{T,r} \rvert = \frac{1}{\underline{\omega}^{3/2}} \sum_{l=1}^m \omega^l \|  q^l_{T} - \bar{q}_{T} \|_1.
 \ead\ee
Combining \eqref{lem:dualtoprimal-proof-eq-3} and \eqref{lem:dualtoprimal-proof-eq-4} gives
\be\label{lem:dualtoprimal-proof-eq-5}
2\|V_{\bar{\pi}} - V_{\tilde{\zeta}}\|_F  \le 2 \bar{c}\left( \frac{1}{\underline{\omega}^{3/2}} + 1  \right) \sum_{l = 1}^m \omega^l  \|q^l_T - \bar{q}_T   \|_1  \le  \left(\frac{1}{\underline{\omega}^{3/2}}+1 \right) \frac{\underline{\omega}^{3/2}\epsilon}{6} \le \frac{\epsilon}{3},
\ee
where the second inequality is due to \eqref{stop-kappa}.
By \eqref{lem:dualtoprimal-proof-eq-1}, \eqref{lem:dualtoprimal-proof-eq-2}, and \eqref{lem:dualtoprimal-proof-eq-5}, we have
\be\label{eq:gradfhatpi}\bad
\|\text{grad}_U f(\hat{\pi}, \hat{U})\|_F &\le \epsilon.
\ead\ee
Combining \eqref{eq:fhatp} and \eqref{eq:gradfhatpi} completes the proof.}
\end{proof}

Now we analyze the iteration complexity of Algorithm \ref{alg:RBCD-IBP}. 
We first present several technical lemmas. The first lemma shows that function $g$ is lower bounded.
{
\begin{lemma}\label{lem:gbound}
Denote $(\pmb{u}^*,\pmb{v}^*,U^*)$ as the global minimum of $g$ defined in \eqref{eq:dualformu}. The following inequality holds:
\be\label{eq:gbound}
g^* :=g(\pmb{u}^*, \pmb{v}^*, U^*) \ge  - \bar{c}/\eta.
\ee
\end{lemma}
\begin{proof}
Notice that at the global minimum, we have for each $\zeta^l$
\be\label{lem:gbound-eq-1}
\|\zeta^l{( (u^l)^*, (v^l)^*, U^*)}\|_1 = \sum_{i,j} \exp{\left(-\frac{1}{\eta}\|(U^*)^\top(x^l_i - y_j)\|^2 + (u^l_i)^* + (v^l_j)^*\right)} = 1,
\ee
which implies that
\be\label{eq:pi-star-zeta-star}
\zeta^l{( (u^l)^*, (v^l)^*, U^*)} = \pi^l{( (u^l)^*, (v^l)^*, U^*)},
\ee
and
\be\label{eq:geq}
g(\pmb{u}^*, \pmb{v}^*, U^*) = \sum_{l = 1}^m \omega^l \left\{ \log\left( \sum_{i,j} [\zeta^l{( (u^l)^*, (v^l)^*, U^*)}]_{ij} \right) - \langle (u^l)^*, p^l \rangle \right\}  =  - \sum_{l = 1}^m \omega^l \langle (u^l)^*, p^l \rangle .
\ee
Notice that $\|C^l\|_\infty \geq \|(U^*)^\top(x^l_i - y_j)\|^2$ for any $i,j$, together with \eqref{lem:gbound-eq-1} we have
\[
\exp{\left(-\frac{1}{\eta}\|C^l\|_\infty + (u^l)^*_i + (v^l)^*_j\right)} \le \exp{\left(-\frac{1}{\eta}\|(U^*)^\top(x^l_i - y_j)\|^2 + (u^l)^*_i + (v^l)^*_j\right)} \leq 1, \forall i, j,
\]
which further implies
\be\label{lem:gbound-eq-2}
(u^l)^*_i + (v^l)^*_j  \le \frac{1}{\eta}\|C^l\|_\infty, \forall i, j.
\ee
Notice that we have $\sum_{l = 1}^m \omega^l (v^l)^* = 0 $, $p^l \in\Delta^n, \forall l \in [m]$ and $q^* = ({\pi^{l}}^*)^\top \onebf \in\Delta^n$, \eqref{lem:gbound-eq-2} indicates that
\[
\sum_{l = 1}^m \omega^l \langle (u^l)^*, p^l \rangle = \sum_{l = 1}^m \omega^l \langle (u^l)^*, p^l \rangle + \sum_{l = 1}^m \omega^l \langle (v^l)^*, q^* \rangle \le \frac{1}{\eta} \sum_{l = 1}^m \omega^l \|C^l\|_\infty \le \frac{\bar{c}}{\eta},
\]
which, combining with \eqref{eq:geq}, yields the desired result.
\end{proof}}

The next two lemmas show that $g(\pmb{u}, \pmb{v}, U)$ is monotonically decreasing in updates \eqref{eq:IBPstep-1} and \eqref{eq:IBPstep-2}.

\begin{lemma}[Decrease of $g$ in $u$]\label{lem:decinu}
Let $\{(u_t, v_t, U_t)\}$ be the sequence generated by Algorithm \ref{alg:RBCD-IBP}. For any $t \ge 0$, the following inequality holds
\be\label{eq:decinu}
g(\pmb{u}_{t+1}, \pmb{v}_t, U_t) - g(\pmb{u}_t, \pmb{v}_t, U_t) \ \le \ 0.
\ee
\end{lemma}
{
\begin{proof}
It is a direct result of \eqref{eq:u-problem}.
\end{proof}}

\begin{lemma}[Decrease of $g$ in $v$]\label{lem:decinv}
Let $\{(u_t, v_t, U_t)\}$ be the sequence generated by Algorithm \ref{alg:RBCD-IBP}. For any $t \ge 0$, the following inequality holds
\be\label{eq:decinv}\bad
g(\pmb{u}_{t+1}, \pmb{v}_{t+1}, U_t) - g(\pmb{u}_{t+1}, \pmb{v}_t, U_t) \ \le \ - \frac{1}{11}\left(\sum_{l = 1}^m \omega^l \| q^l_t - \bar{q}_t\|_1\right)^2,
\ead\ee
where $q^l_t = \pi^l(u^l_{t+1}, v^l_{t}, U_t) \onebf$ and $\bar{q}_t = \sum_{l = 1}^m \omega^l q^l_t $.
\end{lemma}
\begin{proof}
{
Notice that we have
\begin{align}
&g(\pmb{u}_{t+1}, \pmb{v}_{t+1}, U_t) - g(\pmb{u}_{t+1}, \pmb{v}_t, U_t) \nonumber\\
= & \sum_{l = 1}^m \omega^l \left( \log\left(\|\zeta^l(u^l_{t+1}, v^l_{t+1}, U_t)\|_1\right) -  \log\left(\|\zeta^l(u^l_{t+1}, v^l_{t}, U_t)\|_1\right) \right)\nonumber\\
= & \sum_{l = 1}^m \omega^l \log\left(\|q^{t+1}\|_1\right)\nonumber\\
= & \log\left(  \left\langle q^{t+1}, \onebf \right \rangle \right)\nonumber\\
\le & \left\langle q^{t+1}, \onebf \right\rangle - 1 \nonumber\\
= & \langle q_{t+1} - \bar{q}_t, \onebf \rangle\nonumber\\
\le & -\frac{1}{11} \sum_{l = 1}^m \omega^l\| q^l_t - \bar{q}_t\|_1^2,
\end{align}
}
where  the last inequality follows \cite{kroshnin2019complexity}[Lemma 6]. We include its proof here for completeness. Denote $x^- = \max \{-x, 0\}.$ By the definition of $q_{t+1}, $ we have
\begin{align}
\langle q_{t+1} - \bar{q}_t, \onebf \rangle &= \left\langle \exp\left(\sum_{l=1}^m \omega^l \log q_t^l\right) - \bar{q}_t, \onebf \right\rangle \label{long-equation} \\
& \le - \frac{4}{11} \sum_{j=1}^n \frac{1}{[\bar{q}_t]_j} \sum_{l = 1}^m \omega^l \left([q^l_t - \bar{q}_t]_j^-\right)^2 \nonumber\\
& = - \frac{4}{11}  \sum_{l = 1}^m \omega^l  \sum_{j=1}^n \frac{([q^l_t - \bar{q}_t]_j^-)^2}{[\bar{q}_t]_j}\nonumber\\
& \le - \frac{4}{11}  \sum_{l = 1}^m \omega^l  \frac{(\sum_{j=1}^n [q^l_t - \bar{q}_t]_j^-)^2}{ \sum_{j=1}^n [\bar{q}_t]_j}\nonumber\\
& = - \frac{1}{11}  \sum_{l = 1}^m \omega^l\| q^l_t - \bar{q}_t\|_1^2.\nonumber
\end{align}
Where the first inequality in \eqref{long-equation} uses the fact: if $x \in \br_+^m$, $\bar{x} := \sum_{l = 1}^m \omega^l x^l$, then (see the proof later):
\be\label{eq:fact1}\bad
\bar{x} - \prod_{l = 1}^m (x^l)^{\omega^l} \ge \frac{4}{11} \sum_{l = 1}^m \omega^l \frac{\left[(x_l - \bar{x})^-\right]^2}{\bar{x}}.
\ead\ee
The second inequality in \eqref{long-equation} uses the Cauchy-Schwarz inequality. The last equality is based on the fact that $\sum_{j=1}^n [q^l_t - \bar{q}_t]_j^- = - \frac{1}{2} \|q^l_t - \bar{q}_t\|_1$, since $\langle\bar{q}_t, \onebf \rangle = \langle q^l_t , \onebf \rangle = 1. $ Let $D^l = x^l - \bar{x},$ equation \eqref{eq:fact1} can be proved as follows:
\be
    \bar{x} - \prod_{l = 1}^m (x^l)^{\omega^l }
    = \bar{x} - \exp\left\{\sum_{l = 1}^m \omega^l  \ln(\bar{x} + D^l)\right\}
    = \bar{x} \left(1 - \exp\left\{\sum_{l = 1}^m \omega^l \ln\left(1 + \frac{D^l}{\bar{x}}\right)\right\}\right),
\ee
\be
    \sum_{l = 1}^m \omega^l  \ln\left(1 + \frac{D^l}{\bar{x}}\right)
    \le \sum_{l = 1}^m \omega^l  \left(\frac{D^l}{\bar{x}} - \frac{([D^l]^-)^2}{2 \bar{x}^2}\right)
    = -\sum_{l = 1}^m \omega^l  \frac{([D^l]^-)^2}{2 \bar{x}^2}.
\ee
    Notice that $[D^l]^- = \max\{\bar{x} - x_l, 0\} \le \bar{x}$, thus
    $\sum_{l = 1}^m \omega^l  \frac{([D^l]^-)^2}{\bar{x}^2} \le 1$
    and
\be
    \exp\left\{-\frac{1}{2} \sum_{l = 1}^m \omega^l  \frac{([D^l]^-)^2}{\bar{x}^2}\right\}
    \le 1 - \left(1 - e^{-1/2}\right) \sum_{l = 1}^m w_l \frac{([D^l]^-)^2}{\bar{x}^2}
    \le 1 - \frac{4}{11} \sum_{l = 1}^m \omega^l  \frac{([D^l]^-)^2}{\bar{x}^2}.
\ee
This proves \eqref{eq:fact1}.

By Jensen's inequality, we have
\[
\sum_{l = 1}^m \omega^l\| q^l_t - \bar{q}_t\|_1^2 \ge \left(\sum_{l = 1}^m \omega^l\| q^l_t - \bar{q}_t\|_1\right)^2,
\]
which combining with \eqref{long-equation} completes the proof.
\end{proof}

Notice \eqref{eq:RGDstep} is a Riemannian gradient descent step. To prove the objective function $g(\pmb{u}, \pmb{v}, U)$ has sufficient decrease in \eqref{eq:RGDstep}, we first prove the following Lipschitz continuous condition. The proof of Lemma \ref{lem:lipschitz} mainly follows \cite{huang2020riemannian}[Lemma 4.8].

\begin{lemma}\label{lem:lipschitz}
Let $\{(\pmb{u}_t, \pmb{v}_t, U_t)\}$ be the sequence generated by Algorithm \ref{alg:RBCD-IBP}. For any $U\in \M$, we have the following inequality holds:
\[
 g(\pmb{u}_{t+1}, \pmb{v}_{t+1}, U) \le g(\pmb{u}_{t+1}, \pmb{v}_{t+1}, U_t) + \langle \nabla_U g(\pmb{u}_{t+1}, \pmb{v}_{t+1}, U_t), U - U_t \rangle + \frac{\rho}{2} \|U_t - U\|^2_F,
\]
where $\rho = 2\bar{c}/\eta + 4\bar{c}^2/\eta^2.$
\end{lemma}
\begin{proof}
{
For any $\alpha\in[0,1]$, denote $U_\alpha = \alpha U + (1-\alpha) U_t$. Note that $U_\alpha$ is not necessarily on $\M$, though $U\in\M$ and $U_t \in \M$.
Note that $\nabla_Ug(\pmb{u},\pmb{v},U) = -\frac{2}{\eta}V_{\pmb{\pi}{(\pmb{u}, \pmb{v}, U)}}U$. Therefore, we have
\begin{align}\label{lem:lipschitz-proof-eq-1}
&\|\nabla_U g(\pmb{u}_{t+1}, \pmb{v}_{t+1}, U_t) - \nabla_U g(\pmb{u}_{t+1}, \pmb{v}_{t+1}, U_\alpha)\|_F\\
= & \frac{2}{\eta} \| V_{\pmb{\pi}{(\pmb{u}_{t+1}, \pmb{v}_{t+1}, U_t)}} U_t -  V_{\pmb{\pi}{(\pmb{u}_{t+1}, \pmb{v}_{t+1}, U_\alpha)}} U_\alpha\|_F \nonumber\\
= & \frac{2}{\eta} \| V_{\pmb{\pi}{(\pmb{u}_{t+1}, \pmb{v}_{t+1}, U_t)}} U_t - V_{\pmb{\pi}{(\pmb{u}_{t+1}, \pmb{v}_{t+1}, U_t)}} U_\alpha + V_{\pmb{\pi}{(\pmb{u}_{t+1}, \pmb{v}_{t+1}, U_t)}} U_\alpha -  V_{\pmb{\pi}{(\pmb{u}_{t+1}, \pmb{v}_{t+1}, U_\alpha)}} U_\alpha\|_F\nonumber\\
\le & \frac{2}{\eta} \| V_{\pmb{\pi}{(\pmb{u}_{t+1}, \pmb{v}_{t+1}, U_t)}} (U_t - U_\alpha)\|_F + \frac{2}{\eta}\|(V_{\pmb{\pi}{(\pmb{u}_{t+1}, \pmb{v}_{t+1}, U_t)}} -  V_{\pmb{\pi}{(\pmb{u}_{t+1}, \pmb{v}_{t+1}, U_\alpha)}}) U_\alpha\|_F\nonumber\\
\le & \frac{2}{\eta} \| V_{\pmb{\pi}{(\pmb{u}_{t+1}, \pmb{v}_{t+1}, U_t)}} (U_t - U_\alpha)\|_F + \frac{2\alpha}{\eta}\|(V_{\pmb{\pi}{(\pmb{u}_{t+1}, \pmb{v}_{t+1}, U_t)}} -  V_{\pmb{\pi}{(\pmb{u}_{t+1}, \pmb{v}_{t+1}, U_\alpha)}}) U\|_F\nonumber\\
& +\frac{2(1-\alpha)}{\eta}\|(V_{\pmb{\pi}{(\pmb{u}_{t+1}, \pmb{v}_{t+1}, U_t)}} -  V_{\pmb{\pi}{(\pmb{u}_{t+1}, \pmb{v}_{t+1}, U_\alpha)}}) U_t\|_F\nonumber\\
\le & \frac{2}{\eta} \| V_{\pmb{\pi}{(\pmb{u}_{t+1}, \pmb{v}_{t+1}, U_t)}}\|_F \|U_t- U_\alpha\|_F + \frac{2}{\eta}\|V_{\pmb{\pi}{(\pmb{u}_{t+1}, \pmb{v}_{t+1}, U_t)}}  -  V_{\pmb{\pi}{(\pmb{u}_{t+1}, \pmb{v}_{t+1}, U_\alpha)}} \|_F.\nonumber
\end{align}}
By using \eqref{eq:pisolution}, we have
\be\label{eq:Vbound}
\|V_{\pi{(\pmb{u}_{t+1}, \pmb{v}_{t+1},U_t)}}\|_F \le \sum_{l = 1}^m \omega^l  \sum_{i,j} [\pi^l{(u^l_{t+1},v^l_{t+1},U_t)}]_{ij} \|(x^l_i - y_j)(x^l_i - y_j)^\top\|_F \le \max_{l,i,j} \|x^l_i - y_j\|^2 = \bar{c}.
\ee
Note that for fixed $U$, each element of the objective function $f^l_\eta(\pi^l,U) := \sum_{i, j = 1}^n  \pi_{i,j}^l \| U^\top(x^l_i - y_j)\|^2  - \eta H(\pi^l)$ is $\eta$-strongly convex with respect to $\pi^l$ under the $\ell_1$ norm metric, which implies
\begin{align}\label{eq:strongconvexineq}
&f^l_\eta(\pi^l(u^l_{t+1},v^l_{t+1},U_t), U_\alpha) \ge  f^l_\eta( \pi^l(u^l_{t+1},v^l_{t+1},U_\alpha), U_\alpha) + \langle \nabla_{\pi^l} f^l_\eta(\pi^l(u^l_{t+1},v^l_{t+1},U_\alpha), U_\alpha), \nonumber\\
&\pi^l(u^l_{t+1},v^l_{t+1},U_t) - \pi^l(u^l_{t+1},v^l_{t+1},U_\alpha)\rangle + \frac{\eta}{2}\|\pi^l(u^l_{t+1},v^l_{t+1},U_t) - \pi^l(u^l_{t+1},v^l_{t+1},U_\alpha)\|_1^2\\
&f^l_\eta(\pi^l(u^l_{t+1},v^l_{t+1},U_\alpha), U_\alpha) \ge  f^l_\eta( \pi^l(u^l_{t+1},v^l_{t+1},U_t), U_\alpha) + \langle \nabla_{\pi^l} f^l_\eta(\pi^l(u^l_{t+1},v^l_{t+1},U_t), U_\alpha), \nonumber\\
&\pi^l(u^l_{t+1},v^l_{t+1},U_\alpha) - \pi^l(u^l_{t+1},v^l_{t+1},U_t)\rangle + \frac{\eta}{2}\|\pi^l(u^l_{t+1},v^l_{t+1},U_t) - \pi^l(u^l_{t+1},v^l_{t+1},U_\alpha)\|_1^2.\nonumber
\end{align}
By adding the above two inequalities, we have
\begin{align}
& \langle \nabla_{\pi^l} f^l_\eta(\pi^l(u^l_{t+1},v^l_{t+1},U_\alpha),U_\alpha) - \nabla_{\pi^l} f^l_\eta(\pi^l(u^l_{t+1},v^l_{t+1},U_t),U_\alpha) , \pi^l(u^l_{t+1},v^l_{t+1},U_\alpha) - \pi^l(u^l_{t+1},v^l_{t+1},U_t)\rangle \nonumber \\
\ge & \eta\|\pi^l(u^l_{t+1},v^l_{t+1},U_t) - \pi^l(u^l_{t+1},v^l_{t+1},U_\alpha)\|_1^2. \label{eq:primalstronglysum}
\end{align}
Moreover, note that 
{
\be\label{grad-pi-f-eta}
[\nabla_{\pi^l}  f^l_\eta(\pi^l,U)]_{ij} = \|U^\top(x^l_i - y_j)\|^2 + \eta\log(\pi^l_{ij}),
\ee
which, combining with \eqref{eq:innerpi} and \eqref{eq:zeta}, yields
\begin{align*}
&[\nabla_{\pi^l} f^l_\eta(\pi^l(u^l,v^l,U),U)]_{ij} \\
= & \|U^\top(x^l_i - y_j)\|^2 + \eta \log([\pi^l(u^l,v^l,U)]_{ij})\\
=  & \|U^\top(x^l_i - y_j)\|^2 + \eta\left(-\frac{1}{\eta}\|U^\top(x^l_i - y_j)\|^2 + u^l_{i} + v^l_{j} \right) - \eta\log(\|\zeta^l(u^l,v^l,U)\|_1 )\\
=  & \eta(u^l_{i} + v^l_{j}) - \eta\log(\|\zeta^l(u^l,v^l,U)\|_1 ).
\end{align*}
We further compute
\begin{align*}
& \langle \nabla_{\pi^l} f^l_\eta(\pi^l(u^l_{t+1},v^l_{t+1},U_\alpha),U_\alpha) - \nabla_{\pi^l} f^l_\eta(\pi^l(u^l_{t+1},v^l_{t+1},U_t),U_t) , \pi^l(u^l_{t+1},v^l_{t+1},U_\alpha) - \pi^l(u^l_{t+1},v^l_{t+1},U_t)\rangle \nonumber \\
=&  - \eta\left(\log(\|\zeta^l(u^l_{t+1},v^l_{t+1},U_\alpha)\|_1 )- \log(\|\zeta^l(u^l_{t+1},v^l_{t+1},U_t)\|_1 )\right) \langle \onebf , \pi^l(u^l_{t+1},v^l_{t+1},U_\alpha) - \pi^l(u^l_{t+1},v^l_{t+1},U_t)\rangle \nonumber\\
= & 0 .
\end{align*}
Summing the above equality and \eqref{eq:primalstronglysum} yields
\begin{align*}
& \langle \nabla_{\pi^l} f^l_\eta(\pi^l(u^l_{t+1},v^l_{t+1},U_t),U_t) - \nabla_{\pi^l} f^l_\eta(\pi^l(u^l_{t+1},v^l_{t+1},U_t),U_\alpha) , \pi^l(u^l_{t+1},v^l_{t+1},U_\alpha) - \pi^l(u^l_{t+1},v^l_{t+1},U_t)\rangle \nonumber \\
\ge & \eta\|\pi^l(u^l_{t+1},v^l_{t+1},U_t) - \pi^l(u^l_{t+1},v^l_{t+1},U_\alpha)\|_1^2, 
\end{align*} }
which, by H\"{o}lder's inequality, further yields,
\begin{align}\label{eq:inftynorm}
    & \eta\|\pi^l(u^l_{t+1},v^l_{t+1},U_t) - \pi^l(u^l_{t+1},v^l_{t+1},U_\alpha)\|_1 \\
\le & \|\nabla_{\pi^l} f^l_\eta(\pi^l(u^l_{t+1},v^l_{t+1},U_t),U_t) - \nabla_{\pi^l} f^l_\eta(\pi^l(u^l_{t+1},v^l_{t+1},U_t),U_\alpha)\|_\infty \nonumber\\
\le & \max_{i,j}\  \lvert \|(U_\alpha)^\top(x^l_i - y_j)\|_2^2 - \|(U_t)^\top(x^l_i - y_j)\|_2^2  \rvert\nonumber\\
=  & \max_{i,j}\  \lvert (x^l_i - y_j)^\top(U_\alpha (U_\alpha)^\top - U_t(U_t)^\top) (x^l_i - y_j) \rvert \nonumber\\
\le & (\max_{i,j}\  \| x^l_i - y_j\|^2) \|U_\alpha (U_\alpha)^\top - U_t(U_t)^\top\|_F\nonumber\\
=  & \|C^l\|_\infty\|U_\alpha (U_\alpha)^\top - U_t(U_t)^\top\|_F,\nonumber
\end{align}
where the second inequality follows from \eqref{grad-pi-f-eta}. Furthermore, since $U, U_t\in\M$, we have
\be \label{eq:UUTdiff}\bad
   & \|U_\alpha (U_\alpha)^\top - U_t(U_t)^\top\|_F \\
= & \|U_\alpha (U_\alpha)^\top - U_t(U_\alpha)^\top + U_t(U_\alpha)^\top - U_t(U_t)^\top\|_F\\
\le & \|(U_\alpha - U_t)(U_\alpha)^\top\|_F + \|U_t(U_\alpha - U_t)^\top\|_F\\
\le & \|(U_\alpha - U_t)(\alpha U + (1-\alpha) U_t)^\top\|_F + \|U_\alpha - U_t\|_F\\
\le & \alpha\|(U_\alpha - U^t) U^\top\|_F + (1-\alpha) \|(U_\alpha - U_t) (U_t)^\top\|_F + \|(U_\alpha - U_t)\|_F\\
=  & 2\|U_\alpha - U_t\|_F.
\ead\ee
By combining \eqref{eq:inftynorm} and \eqref{eq:UUTdiff}, we have
\begin{align}\label{eq:bound-V-pi-diff}
& \|V_{\pmb{\pi}{(\pmb{u}_{t+1}, \pmb{v}_{t+1}, U_t)}}  -  V_{\pmb{\pi}{(\pmb{u}_{t+1}, \pmb{v}_{t+1}, U_\alpha)}} \|_F\\
\leq & \sum_{l = 1}^m\omega^l \left[ \sum_{i,j}^n \left\lvert \pi^l{(u^l_{t+1},v^l_{t+1},U_t)}_{i,j} - \pi^l{(u^l_{t+1},v^l_{t+1},U_\alpha)}_{i,j} \right\rvert \cdot \|(x^l_i - y_j)(x^l_i - y_j)^T \|_F \right]\nonumber\\
\leq & \sum_{l = 1}^m\omega^l \|C^l\|_\infty  \|\pi^l{(u^l_{t+1},v^l_{t+1},U_t)} - \pi^l{(u^l_{t+1},v^l_{t+1},U_\alpha)}\|_1 \nonumber\\
\leq &  \sum_{l = 1}^m\omega^l  \frac{2\|C^l\|^2_\infty}{\eta} \|U_t - U_\alpha \|_F \leq \frac{2 \bar{c}^2}{\eta} \sum_{l = 1}^m\omega^l \|U_t - U_\alpha \|_F =  \frac{2 \bar{c}^2}{\eta}\|U_t - U_\alpha \|_F.\nonumber
\end{align}
Plugging \eqref{eq:Vbound} and \eqref{eq:bound-V-pi-diff} into \eqref{lem:lipschitz-proof-eq-1}  yields:
\begin{align*}
\|\nabla_U g(\pmb{u}_{t+1}, \pmb{v}_{t+1}, U_t) - \nabla_U g(\pmb{u}_{t+1}, \pmb{v}_{t+1}, U_\alpha)\|_F \le \rho \|U_t - U_\alpha\|_F.
\end{align*}
We then have
\be\label{eq:lipineq} \bad
&\lvert g(\pmb{u}_{t+1}, \pmb{v}_{t+1}, U) - g(\pmb{u}_{t+1}, \pmb{v}_{t+1}, U_t) - \langle \nabla_U g(\pmb{u}_{t+1}, \pmb{v}_{t+1}, U_t), U - U_t \rangle \rvert\\
= & \left\lvert \int_0^1 \langle \nabla_U g(\pmb{u}_{t+1}, \pmb{v}_{t+1}, \alpha U + (1-\alpha) U_{t}) - \nabla_U g(\pmb{u}_{t+1}, \pmb{v}_{t+1}, U_t), U - U_t \rangle d\alpha  \right\rvert\\
\le & \int_0^1 \| \nabla_U g(\pmb{u}_{t+1}, \pmb{v}_{t+1}, \alpha U + (1-\alpha) U_{t}) - \nabla_U g(\pmb{u}_{t+1}, \pmb{v}_{t+1}, U_t)\|_F \|U - U_t\|_F d\alpha\\
\le & \int_0^1 \rho \alpha \|U - U_t\|_F^2 d\alpha\\
= & \frac{\rho}{2} \|U - U_t\|_F^2,
\ead
\ee
which completes the proof.
\end{proof}

We now prove that function $g$ is decreasing after updating $U.$
\begin{lemma}[Decrease of $g$ in $U$]\label{lem:decU}
Let $\{(u^t, v^t, U^t)\}$ be the sequence generated by Algorithm \ref{alg:RBCD-IBP}. For any $t \ge 0$, the following inequality holds
\be\label{decrease-U}
g(\pmb{u}_{t+1}, \pmb{v}_{t+1}, U_{t+1}) - g(\pmb{u}_{t+1}, \pmb{v}_{t+1}, U_t ) \le - \frac{1}{8L_2\bar{c}/\eta + 2\rho L_1^2 } \|\xi_{t+1}\|^2_F,
\ee
where $\rho$ is defined in Lemma \ref{lem:lipschitz}, $L_1$ and $L_2$ are defined in Proposition \ref{pro:retr}, and $\xi_{t+1}:=\grad_U g(\pmb{u}_{t+1}, \pmb{v}_{t+1}, U_t)$.
\end{lemma}
\begin{proof}
By setting $U = U_{t+1} = \retr_{U_t}(- \tau \xi_{t+1})$ in Lemma \ref{lem:lipschitz}, we have,
\be\label{eq:lipineq-1} \bad
 g(\pmb{u}_{t+1}, \pmb{v}_{t+1}, U_{t+1}) - g(\pmb{u}_{t+1}, \pmb{v}_{t+1}, U_{t}) &\le \langle \nabla_U g(\pmb{u}_{t+1}, \pmb{v}_{t+1}, U_{t}), U_{t+1} - U_t \rangle  +  \frac{\rho}{2} \|U_{t+1} - U_t\|_F^2 \\
& \le \langle \nabla_U g(\pmb{u}_{t+1}, \pmb{v}_{t+1}, U_{t}), U_{t+1} - U_t \rangle + \frac{\rho\tau^2L_1^2}{2}\|\xi_{t+1}\|_F^2,
\ead\ee
where the last inequality follows from Proposition \ref{pro:retr}.
We then have
\be\label{eq:lipinnerterm}\bad
   & \langle \nabla_U g(\pmb{u}_{t+1}, \pmb{v}_{t+1}, U_{t}), U_{t+1} - U_t \rangle \\
= & \langle \nabla_U g(\pmb{u}_{t+1}, \pmb{v}_{t+1}, U_{t}),  \retr_{U_t}(- \tau \xi_{t+1}) - U_t \rangle \\
= & \langle \nabla_U g(\pmb{u}_{t+1}, \pmb{v}_{t+1}, U_{t}),  - \tau \xi_{t+1} \rangle +  \langle \nabla_U g(\pmb{u}_{t+1}, \pmb{v}_{t+1}, U_{t}),  \retr_{U_t}(- \tau \xi_{t+1}) - (U_t - \tau \xi_{t+1}) \rangle \\
\le & - \tau \langle \nabla_U g(\pmb{u}_{t+1}, \pmb{v}_{t+1}, U_{t}),  \xi_{t+1} \rangle +  \| \nabla_U g(\pmb{u}_{t+1}, \pmb{v}_{t+1}, U_{t})\|_F\| \retr_{U_t}(- \tau \xi_{t+1}) - (U_t - \tau \xi_{t+1} )\|_F\\
\le & - \tau \langle \nabla_U g(\pmb{u}_{t+1}, \pmb{v}_{t+1}, U_{t}),  \xi_{t+1} \rangle +  \frac{2}{\eta} \|V_{\pmb{\pi}{(\pmb{u}_{t+1}, \pmb{v}_{t+1}, U_{t})}}U^t\|_F \cdot L_2 \tau^2 \|\xi_{t+1}\|_F^2\\
\le &  - \tau \langle \nabla_U g(\pmb{u}_{t+1}, \pmb{v}_{t+1}, U_{t}),  \xi_{t+1} \rangle +  \frac{2}{\eta}  L_2 \tau^2 \bar{c} \|\xi_{t+1}\|_F^2\\
=  & - \tau \| \xi_{t+1}\|_F^2 +  \frac{2}{\eta}  L_2 \tau^2\bar{c} \|\xi_{t+1}\|_F^2,
\ead\ee
where the second inequality follows from Proposition \ref{pro:retr}, and the last inequality is due to \eqref{eq:Vbound}.
Combining \eqref{eq:lipineq-1} and \eqref{eq:lipinnerterm} yields,
\[
g(\pmb{u}_{t+1}, \pmb{v}_{t+1}, U_{t+1}) - g(\pmb{u}_{t+1}, \pmb{v}_{t+1}, U_{t} )
\leq  -\tau\left(1 - \left(\frac{2}{\eta}  L_2 \bar{c}  + \frac{\rho}{2}L_1^2\right) \tau\right)\|\xi_{t+1}\|_F^2.
\]
Finally, choosing $\tau = \frac{1}{4  L_2 \bar{c}/\eta  + \rho L_1^2}$ gives the desired result \eqref{decrease-U}.
\end{proof}

We now prove Theorem \ref{thm:RBCDmain}.

\begin{proof}
By combining Lemmas \ref{lem:decU}, \ref{lem:decinu} and \ref{lem:decinv}, we have:
\begin{align}\label{thm:dualmain-eq-1}
& g(\pmb{u}_{t+1}, \pmb{v}_{t+1}, U_{t+1}) - g(\pmb{u}_{t}, \pmb{v}_{t}, U_t)\\
\le & - \left(\frac{1}{11} \left(\sum_{l = 1}^m \omega^l \| q^l_t - \bar{q}_t\|_1\right)^2 + \frac{1}{8L_2\bar{c}/\eta + 2\rho L_1^2 } \|\xi_{t+1}\|^2_F\right).\nonumber
\end{align}
{
Suppose Algorithm \ref{alg:RBCD-IBP} terminates at the $T$-th iteration. Summing \eqref{thm:dualmain-eq-1} over $t=0,\ldots,T-1$ yields
\begin{align}\label{eq:telescoping}
    & g(\pmb{u}_{T}, \pmb{v}_{T}, U_{T}) -  g(\pmb{u}_{0}, \pmb{v}_{0}, U_{0}) \\
\le & - \sum_{t= 0}^{T-1} \left(\frac{1}{11} \left(\sum_{l = 1}^m \omega^l \| q^l_t - \bar{q}_t\|_1\right)^2 + \frac{1}{8L_2\bar{c}/\eta + 2\rho L_1^2 } \|\xi_{t+1}\|^2_F\right)\nonumber \\
= &  - \sum_{t= 0}^{T-1} \left(\frac{1}{11} \left(\sum_{l = 1}^m \omega^l \| q^l_t - \bar{q}_t\|_1\right)^2  + \frac{\eta^2\|\xi^{t+1}\|^2_F}{(8L_2\bar{c} + 4 L_1^2\bar{c})\eta + 8L_1^2\bar{c}^2  } \right)\nonumber \\
\le &  - \sum_{t= 0}^{T-1} \min\left\{\frac{1}{11},\frac{1}{(8L_2\bar{c}+ 4 L_1^2\bar{c})\eta + 8L_1^2\bar{c}^2  }\right\}\cdot \left(\left(\sum_{l = 1}^m \omega^l \| q^l_t - \bar{q}_t\|_1\right)^2   + \eta^2\|\xi^{t+1}\|^2_F\right) \nonumber \\
\le & -T \cdot \min\left\{\frac{1}{11},\frac{1}{(8L_2\bar{c} + 4 L_1^2\bar{c})\eta + 8L_1^2\bar{c}^2}\right\}\cdot \min \left\{ \frac{\underline{\omega}^{3}\epsilon^2}{144\bar{c}^2},  \frac{\epsilon^2}{9}\right\}, \nonumber
\end{align}
where the equality is obtained by plugging in the definition of $\rho$ in \eqref{thm:RBCD-param}, and the last inequality follows from the fact that the stopping criteria in Algorithm \ref{alg:RBCD-IBP} does not hold for $t<T$.
By combining with \eqref{eq:gbound} and \eqref{thm:RBCD-param}, \eqref{eq:telescoping} immediately leads to
\begin{align}\label{thm:dualmain-eq-2}
T & \le (g( \pmb{u}_{0}, \pmb{v}_{0}, U_{0}) - g^*) \cdot\max\left\{11,(8L_2\bar{c} + 4 L_1^2\bar{c})\eta + 8L_1^2\bar{c}^2 \right\} \cdot\max\left\{ \frac{144\bar{c}^2}{\underline{\omega}^{3}\epsilon^2},  \frac{9}{\epsilon^2}\right\}  \\
& \le \left(g( u_{0}, v_{0}, U_{0}) - 1 + \frac{ \bar{c}}{\eta}\right) \cdot\max\left\{11, (8L_2\bar{c} + 4 L_1^2\bar{c})\eta + 8L_1^2\bar{c}^2 \right\} \cdot\max\left\{ \frac{144\bar{c}^2}{\underline{\omega}^{3}\epsilon^2},  \frac{9}{\epsilon^2}\right\} \nonumber \\
& = O\left(\frac{L_1^2 \bar{c}^5\log(n)}{\underline{\omega}^{3}\epsilon^3} \right), \nonumber
\end{align}
where $g^*$ is defined in \eqref{eq:gbound}. This completes the proof of Theorem \ref{thm:RBCDmain}.}
\end{proof}

\paragraph{Proof of Corollary \ref{cor:RBCDmain}.} We further analyze the per-iteration complexity for Algorithm \ref{alg:RBCD-IBP}. Notice that in each iteration, we need to compute the projected cost matrices $\{M^l\}_{l \in [m]}$, which takes $O(mn^2dk)$ arithmetic operations. Secondly, steps \eqref{eq:IBPstep-1} - \eqref{eq:IBPstep-2} can be done in $O(mn + mn^2)$ arithmetic operations. Moreover, the retraction operation requires $O(dk^2 + k^3)$ arithmetic operations and the complexity of computing $V_\pi U$ is $O(mn^2dk)$. Therefore, the per-iteration arithmetic operations complexity of Algorithm \ref{alg:RBCD-IBP} is
$$O(mn^2dk + mn + mn^2 + dk^2 + k^3 + mn^2dk ),$$
which combining with Theorem \ref{thm:RBCDmain} proves Corollary \ref{cor:RBCDmain}.

\section{Proof of Theorem \ref{thm:RGAmain}}\label{sec:rgathm}

The proof of Theorem \ref{thm:RGAmain} includes two parts. We first show that the objective function $f_\eta(U)$ defined in \eqref{eq:fU} is monotonically increasing in Algorithm \ref{alg:RGA-IBP}, which leads to the iteration complexity for obtaining an $\epsilon$-stationary point. We then analyze the complexity of the WB subproblem in each iteration. The following lemma shows $f_\eta(U)$ is Lipschitz continuous.

\begin{lemma}\label{lem:feta-lipschitz}
For any $U_1, U_2\in \St(d, k)$, we have the following inequality holds:
\[
\lvert f_\eta(U_1) - f_\eta( U_2) - \langle \nabla f_\eta( U_2), U_1 - U_2 \rangle\rvert \le \frac{\rho}{2} \|U_1 - U_2\|^2_F,
\]
where $\rho = 2\bar{c}+ \frac{4\bar{c}^2}{\eta}.$
\end{lemma}
\begin{proof}
The proof of this lemma mainly follows the proof of \cite{lin2020projection}[Lemma 3.2] and Lemma \ref{lem:lipschitz}. Denote $U_\alpha = \alpha U_1 + (1-\alpha) U_2$. The gradient of $f_\eta(U)$ is $\nabla  f_\eta(U) = 2V_{\pmb{\pi}_\eta^*(U)}U$, indicating
\begin{align}\label{eq:f-eta-lipschitz-proof-eq-1}
&\|\nabla f_\eta(U_\alpha) - \nabla f_\eta(U_2)\|_F\le 2 \| V_{\pmb{\pi}_\eta^*(U_\alpha)}\|_F \|U_\alpha - U_2\|_F + 2\| V_{\pmb{\pi}_\eta^*(U_\alpha)} - V_{\pmb{\pi}_\eta^*(U_2)}\|_F,
\end{align}
where $\pmb{\pi}_\eta^*(U)$ denotes the optimal solution of \eqref{eq:fU}.
Notice that  $\sum_{l = 1}^m\omega^l \sum_{i,j}^n (\pi_\eta^*(U))^l_{i,j} =1 $, we have
\be\label{eq:Vbound-RGA}
\|V_{\pmb{\pi}_\eta^*(U_\alpha)}\|_F \le \sum_{l = 1}^m \omega^l  \sum_{i,j} (\pi_\eta^*(U_\alpha))^l_{i,j} \|(x^l_i - y_j)(x^l_i - y_j)^\top\|_F \le \max_{l,i,j} |x^l_i - y_j|^2 = \max_l \|C^l\|_\infty.
\ee
Following the idea of \eqref{eq:strongconvexineq} and \eqref{eq:primalstronglysum} we have
\begin{align}\label{eq:f-eta-lipschitz-proof-eq-2}
\langle \nabla_{\pi^l} f^l_\eta( (\pi_\eta^*(U_\alpha))^l ,U_\alpha) - \nabla_{\pi^l} f^l_\eta( (\pi_\eta^*(U_2))^l ,U_\alpha), (\pi_\eta^*(U_\alpha))^l - (\pi_\eta^*(U_2))^l \rangle \ge \eta\|(\pi_\eta^*(U_\alpha))^l - (\pi_\eta^*(U_2))^l\|_1^2.
\end{align}
By first order optimality condition of \eqref{eq:fU}, we have
\begin{align}
&\langle \nabla_{\pi^l} f^l_\eta( (\pi_\eta^*(U_\alpha))^l ,U_\alpha) ,  (\pi_\eta^*(U_2))^l - (\pi_\eta^*(U_\alpha))^l \rangle \ge 0, \nonumber\\
&\langle \nabla_{\pi^l} f^l_\eta( (\pi_\eta^*(U_2))^l ,U_2) ,  (\pi_\eta^*(U_\alpha))^l - (\pi_\eta^*(U_2))^l \rangle \ge 0,
\end{align}
which leads to
\begin{align} \label{eq:f-eta-lipschitz-proof-eq-3}
\langle \nabla_{\pi^l} f^l_\eta( (\pi_\eta^*(U_2))^l ,U_2) - \nabla_{\pi^l} f^l_\eta( (\pi_\eta^*(U_\alpha))^l ,U_\alpha), (\pi_\eta^*(U_\alpha))^l - (\pi_\eta^*(U_2))^l \rangle \ge 0.
\end{align}
By adding \eqref{eq:f-eta-lipschitz-proof-eq-2} and \eqref{eq:f-eta-lipschitz-proof-eq-3} together, we have
\begin{align} \label{eq:f-eta-lipschitz-proof-eq-4}
\langle \nabla_{\pi^l} f^l_\eta( (\pi_\eta^*(U_2))^l ,U_2) - \nabla_{\pi^l} f^l_\eta( (\pi_\eta^*(U_2))^l ,U_\alpha), (\pi_\eta^*(U_\alpha))^l - (\pi_\eta^*(U_2))^l \rangle \ge \eta\|(\pi_\eta^*(U_\alpha))^l - (\pi_\eta^*(U_2))^l\|_1^2.
\end{align}
The rest of the proof follows \eqref{eq:inftynorm} - \eqref{eq:lipineq}.
\end{proof}

The next lemma proves equation \eqref{eq:ibpstopcondition}.

\begin{proof} 
Notice that when fixing $U_t$, each $f^l_\eta(\pi^l,U) = \sum_{i, j = 1}^n  \pi_{i,j}^l \| U^\top(x^l_i - y_j)\|^2  - \eta H(\pi^l)$ is $\eta$-strongly convex with respect to $\pi^l$ under the $\ell_1$ norm metric, which implies
\be\bad
&f^l_\eta(\pi^l, U_t) \ge  f^l_\eta( (\pi_\eta^*)^l, U_t) + \frac{\eta}{2} \|\pi^l - (\pi_\eta^*)^l\|_1^2.
\ead\ee
In the $j$-th iteration of the IBP subroutine (Algorithm \ref{alg:IBP}), we have
\be\label{eq:boundpidiff}\bad
\frac{\eta}{2}\left( \sum_{l=1}^m \omega^l \| \pi^l_j - (\pi_\eta^*)^l \|_1\right)^2 & \le  \frac{\eta}{2}\sum_{l=1}^m \omega^l \| \pi^l_j - (\pi_\eta^*)^l \|_1^2\le  f_\eta(\pmb{\pi}_j, U_t) -   f_\eta( \pmb{\pi}_\eta^*, U_t)\\
&= -\eta (g(\pmb{u}_j, \pmb{v}_j, U_t) - g(\pmb{u}_\eta^*, \pmb{v}_\eta^*, U_t)) + \eta\sum_{l=1}^m \omega^l \langle v_j^l, q^l\rangle \\
& \le \eta\sum_{l=1}^m \omega^l \langle v_j^l, q^l\rangle,
\ead\ee
where the first inequality is by Jensen's inequality, and the last equality uses Lemma \ref{lem:primaldualrelation} and the fact that
\be
f_\eta( \pmb{\pi}_\eta^*, U_t) = -\eta  g(\pmb{u}_\eta^*, \pmb{v}_\eta^*, U_t),
\ee
which can be proved by plugging $\pmb{\pi}_\eta^*, \pmb{u}_\eta^*, \pmb{v}_\eta^*$ into Lemma \ref{lem:primaldualrelation} and $\sum_{l=1}^m \omega^l (v_\eta^*)^l = 0.$ Using the similar idea in \eqref{eq:fetagrelation}, we bound the term $\sum_{l=1}^m \omega^l \langle v^l, q^l\rangle$ as follows.
\be\label{eq:boundvdotq}\bad
\sum_{l=1}^m \omega^l \langle v_j^l, q^l_j \rangle&=  \sum_{l=1}^m \omega^l \langle v^l _j, q^l_j - \bar{q}_j \rangle=  \sum_{l=1}^m \omega^l \langle v^l _j - b^l_j \onebf, q^l_j - \bar{q}_j \rangle\\
& \le  \sum_{l=1}^m \omega^l  \| v^l_j - b^l_j\onebf\|_\infty\| q^l_j - \bar{q}_j \|_1\le R \sum_{l=1}^m \omega^l \| q^l_j - \bar{q}_j \|_1,
\ead\ee
where $b^l_j = \frac{\min_i [v^l _j]_i + \max_i [v^l _j]_i}{2}$, $R = \bar{c}/\eta.$ The first equality in \eqref{eq:boundvdotq} comes from $\sum_{l=1}^m \omega^l  v^l = 0,$ the second equality in \eqref{eq:boundvdotq} is due to $\langle q^l_j, \onebf\rangle = 1,  \langle \bar{q}_j, \onebf\rangle = 1$, and the last inequality in \eqref{eq:boundvdotq} is by H\"{o}lder's inequality and \cite{kroshnin2019complexity}[Lemma 4].
Therefore, if the Algorithm \ref{alg:IBP} terminates at the $t$-th iteration, then combining \eqref{eq:boundpidiff} with \eqref{eq:boundvdotq} yields
\be\label{eq:boundpidiff-1}\bad
\left( \sum_{l=1}^m \omega^l \| \pi^l_t - (\pi_\eta^*)^l \|_1\right)^2  \le 2 \sum_{l=1}^m \omega^l \langle v_t^l, q_t^l\rangle \le 2R \sum_{l=1}^m \omega^l \| q^l_t - \bar{q}_t \|_1 \le 2R\cdot \frac{\epsilon^2}{200\bar{c}^2R} = \frac{\epsilon^2}{100\bar{c}^2},
\ead\ee
which leads to \eqref{eq:ibpstopcondition}.
\end{proof}

The next lemma shows the iteration complexity for obtaining $U_T$ that satisfies $\|\xi_{T}\| \le \epsilon.$

\begin{lemma} \label{lem:rgaitercomplexity}
Choose parameters as in \eqref{lem:dualmain-param}.
The Algorithm \ref{alg:RGA-IBP} terminates in $T$ iterations, where $T$ is defined in \eqref{lem:dualmain-T}.
\end{lemma}
\begin{proof}
By Lemma \ref{lem:feta-lipschitz} and the definition of $U_{t+1}$, we have
\be\label{eq:proof-rga-eq-1}\bad
f_\eta(U_{t+1}) - f_\eta(U_{t})  &\ge \langle \nabla f_\eta( U_t), U_{t+1} - U_t \rangle - \frac{\rho}{2} \|U_{t+1} - U_t\|^2_F\\
& = \langle \nabla f_\eta( U_t), \tau \xi_{t+1} \rangle + \langle \nabla f_\eta( U_t), \retr_{U_t}(\tau \xi_{t+1}) - (U_t + \tau \xi_{t+1})\rangle - \frac{\rho}{2} \|U_{t+1} - U_t\|^2_F\\
&\ge \tau \langle \nabla f_\eta( U_t), \xi_{t+1} \rangle - \| \nabla f_\eta( U_t)\|_F \|\retr_{U_t}(\tau \xi_{t+1}) - (U_t + \tau \xi_{t+1})\|_F - \frac{\rho}{2} \|U_{t+1} - U_t\|^2_F\\
&\ge\tau  \langle \text{grad} f_\eta( U_t), \xi_{t+1} \rangle - 2\tau^2L_2 \bar{c}\|\xi_{t+1}\|_F^2 - \frac{1}{2}\rho\tau^2L_1^2\|\xi_{t+1}\|_F^2,
\ead\ee
where the last inequality uses Proposition \ref{pro:retr} and $\| \nabla f_\eta( U_t)\|_F = \|2V_{\pmb{\pi}_\eta^*(U_t)}U_t\|_F\le 2\bar{c}.$ We further bound \eqref{eq:proof-rga-eq-1} by the following inequalities:
\be\label{eq:proof-rga-eq-2}\bad
\langle \text{grad} f_\eta( U_t), \xi_{t+1} \rangle \ge \frac{1}{2}(\| \xi_{t+1}\|_F^2 - \| \xi_{t+1} -  \text{grad} f_\eta( U_t)\|_F^2).
\ead\ee
Combining \eqref{eq:proof-rga-eq-1} - \eqref{eq:proof-rga-eq-2} yields
\be\label{eq:proof-rga-eq-4}
f_\eta(U_{t+1}) - f_\eta(U_{t})  \ge \tau \left[\frac{1}{2} - \tau\left(2L_2\bar{c}+ \frac{1}{2}\rho L_1^2\right)\right] \| \xi_{t+1} \|_F^2 - \frac{1}{2}\tau  \| \xi_{t+1} - \text{grad} f_\eta( U_t) \|_F^2.
\ee
Notice that the IBP subroutine in each iteration of Algorithm \ref{alg:RGA-IBP} returns $\pmb{\pi}_{t+1}$ satisfying \eqref{eq:ibpstopcondition}. Therefore, we have
\be\label{eq:proof-rga-eq-5}\bad
 \| \xi_{t+1} -  \text{grad} f_\eta( U_t)\|_F &= 2\|(V_{\pmb{\pi}_{t+1}} - V_{\pmb{\pi}_\eta^*(U_t)})U_t\|_F  \le 2\|V_{\pmb{\pi}_{t+1}} - V_{\pmb{\pi}_\eta^*(U_t)}\|_F\\
 &\le 2\bar{c}\sum_{l =1}^m \omega^l \| \pi^l_{t+1} -  \pi_\eta^*(U_t)^l \|_1\le \frac{\epsilon}{5}.
\ead\ee
Plugging $\tau = \frac{1}{8L_2\bar{c} + 2\rho L_1^2 }$ into \eqref{eq:proof-rga-eq-4} and combining with \eqref{eq:proof-rga-eq-5}, we have
\be\label{eq:proof-rga-eq-6}\bad
f_\eta(U_{t+1}) - f_\eta(U_{t})  \ge & \frac{\tau}{4} \| \xi_{t+1} \|_F^2- \frac{\tau \epsilon^2}{50}.
\ead\ee
Assume Algorithm \ref{alg:RGA-IBP} stops at the $T$-th iteration. For any $t < T$, we have  $\|\xi_{t+1}\|_F > \epsilon.$ Summing \eqref{eq:proof-rga-eq-6} over $t = 0,\ldots,T-1$ yields
\be\label{eq:proof-rga-eq-8}\bad
f_\eta^* - f_\eta(U_{0})  \ge f_\eta(U_{T}) - f_\eta(U_{0})  \ge & T \cdot \frac{23\tau \epsilon^2}{100},
\ead\ee
where $f_\eta^*$ denotes the maximal value of $f_\eta$.
By Lemma \ref{lem:feta-lipschitz}, we have
\be\label{eq:proof-rga-eq-9}\bad
f_\eta^* - f_\eta(U_{0}) & \le \langle \nabla f_\eta( U^*), U_0 - U^* \rangle +   \frac{\rho}{2} \|U^* - U_0\|_F^2  \le \|\nabla f_\eta( U^*)\|_F\|U_0 - U^* \|_F  +  k\rho \le 4\sqrt{k}  \bar{c}+ k\rho,
\ead\ee
where the last inequality comes from $\nabla  f_\eta(U) = 2V_{\pmb{\pi}_\eta^*(U)}U$ and \eqref{eq:Vbound-RGA}. Combining \eqref{eq:proof-rga-eq-8} with \eqref{eq:proof-rga-eq-9} and the definition of $\tau, \rho$ yields
\be\bad
T \le \frac{100(4\sqrt{k}  \bar{c}+ k\rho)(8L_2\bar{c} + 2L_1^2\rho)}{23\epsilon^2} = O\left(\frac{k\log(n)^2L_1^2\bar{c}^4}{\epsilon^4}\right).
\ead\ee
This completes the proof.
\end{proof}

\paragraph{Proof of Theorem \ref{thm:RGAmain}.}
\begin{proof}
By Lemma \ref{lem:rgaitercomplexity}, we have the iteration complexity of Algorithm \ref{alg:RGA-IBP}. The rest of the proof is to show that when Algorithm \ref{alg:RGA-IBP} stops, $(\hat{\pmb{\pi}},\hat{U})$ is an $\epsilon$-stationary point of \eqref{eq:PRWBdiscrete}. We first notice that the stopping criteria guarantees
\[
\|\text{grad}_U f(\hat{\pmb{\pi}}, \hat{U})\|_F = \|\proj_{\Tg_U\St}(2V_{\pmb{\pi}_{T+1}}U_T)\|_F = \|\xi_{T+1}\|_F \le \epsilon,
\]
which verifies \eqref{def:primalsta-eq-1}.
Secondly, $\pmb{\pi}_\eta^*$ is the optimal solution of the regularized WB problem and $\eta = \frac{\epsilon}{4\log(n)+2}$, we have
\be\bad
\langle U_TU_T^\top, V_{\pmb{\pi}_\eta^*(U_T)} \rangle & \le \langle U_TU_T^\top,   V_{\pmb{\pi}^*(U_T)}\rangle - \eta \sum_{l = 1}^m \omega^lH(\pi^*(U_T)^l)
 +   \eta \sum_{l = 1}^m \omega^l H(\pi_\eta^*(U_T)^l)\\
 & \le \langle U_TU_T^\top,   V_{\pmb{\pi}^*(U_T)}\rangle + \frac{\epsilon}{2},
\ead\ee
where in the last step we use $0 \le H(\pi) \le 2 \log(n) + 1.$ By the stopping criteria of the IBP subroutine, we have
\be\bad
0 \le \langle U_TU_T^\top, V_{\pmb{\pi}_{T+1}}  - V_{\pmb{\pi}_\eta^*(U_T)} \rangle &  \le \bar{c} \sum_{l=1}^m\omega^l \|\pi^l_{T+1} -\pi_\eta^*(U_T)^l \|_1 \le \frac{\epsilon}{10}.
\ead\ee
Adding the above two inequalities shows that \eqref{def:primalsta-eq-2} holds, which completes the proof.
\end{proof}

\paragraph{Proof of Corollary \ref{cor:RGAmain}.}
\begin{proof}
We first analyze the per-iteration complexity of Algorithm \ref{alg:RGA-IBP}. The computation of $\{M^l\}_{l \in [m]}$, the retraction and $V_\pi U$ requires $O(mn^2d)$, $O(dk^2 + k^3)$ and $O(mn^2dk)$ arithmetic operations.
By \cite{kroshnin2019complexity}[Theorem 1], it takes $O(mn^2\epsilon'^{-2})$ arithmetic operations for the IBP algorithm to satisfy $\sum_{l=1}^m \omega^l \|q^l_j - \bar{q}_j\|_1 \le \epsilon'$. In our case, we need to bound $\sum_{l=1}^m \omega^l \|q^l_j - \bar{q}_j\|_1$ by
\be\label{eq:bounddiffq}\bad
\eta\epsilon^2/(200\bar{c}^3) = O\left(\frac{\epsilon^3}{\bar{c}^3\log(n)}\right).
\ead\ee
Therefore, it takes
$$O\left(mn^2\frac{\bar{c}^6\log(n)^2}{\epsilon^6}\right)$$
arithmetic operations to satisfy $\sum_{l=1}^m \omega^l \| \pi^l_j - (\pi_\eta^*)^l \|_1 \le \epsilon/ (10\bar{c})$. The total per-iteration complexity of Algorithm \ref{alg:RGA-IBP} is
$$O\left(mn^2d + dk^2 + k^3 + mn^2dk +  mn^2\frac{\bar{c}^6\log(n)^2}{\epsilon^6}\right),$$
which, together with \eqref{lem:dualmain-T} gives the total arithmetic operations complexity given in Corollary \ref{cor:RGAmain}.
\end{proof}

\section{Additional Details for Numerical Experiments}
We provide more details for the numerical experiments in Section \ref{sec:experiment}.

\subsection{Wasserstein Barycenter of Gaussian distributions}
The ground truth Wasserstein Barycenter of a set of Gaussian distributions can be computed by an iterative method introduced in \cite{alvarez2015note}, specifically the following theorem.

\begin{theorem}[\cite{alvarez2015note}] \label{thm:gaussiancenter} Assume $\Sigma^1 , ..., \Sigma^m$ are symmetric $d \times d$ positive semidefinite matrices, with at least one of them positive definite. Consider some symmetric, positive definite $S_0$ and define
\be \label{eq:iterative} \bad
S_{n+1} = S_{n}^{-1/2} (\sum_{l = 1}^m \omega^l \left( S_n^{1/2} \Sigma^l S_n^{1/2} \right)^{1/2} )^2 S_{n}^{-1/2}, n \ge 0 .
\ead\ee
If $\mathcal{N}(0, \Sigma^0)$ is the barycenter of $\mathcal{N}(0, \Sigma^1),...,\mathcal{N}(0, \Sigma^m)$, then
$$W_2(\mathcal{N}(0, S_n),\mathcal{N}(0, \Sigma^0)) \to  0$$
as $n \to \infty$. Furthermore, the barycenter value can be computed as
$$WB(\{\mu^l\}_{l = 1}^m) = Tr(\Sigma^0) + \sum_{l = 1}^m \omega^l Tr(\Sigma^l) - 2  \sum_{l = 1}^m \omega^l Tr( ( (\Sigma^0)^{1/2} \Sigma^l (\Sigma^0)^{1/2})^{1/2}  ).$$
\end{theorem}

\subsection{D2 Clustering and Projected D2 Clustering}

We present the D2 clustering and Projected D2 clustering in Algorithm \ref{alg:d2}. Specially, Algorithm \ref{alg:d2} with Option 1 gives D2 clustering, and Algorithm \ref{alg:d2} with Option 2 gives PD2 clustering. The D2 clustering follows the idea of k-means clustering but clusters discrete distributions under the Wasserstein metric. In each iteration, the D2 Clustering algorithm calculates the Wasserstein distance between each distribution and each barycenter and relabel the distributions. Based on the updated labels, we recalculate the Wasserstein Barycenter for each cluster.

\begin{algorithm}[htb]
\caption{(Projected) D2-clustering}
\label{alg:d2}
\begin{algorithmic}[1]
\STATE \textbf{Input:} $\{\mu_n^i = (p^i, X^i)\}_{i \in [N]}$, $K$
\STATE Initialize the labels for each distribution, denoted as $\text{labels}[i], i\in [N]$;
\STATE Choosing $K$ random distributions as the barycenters $\{ Q_i = (q_i, Y_i), i \in [K]\}$;
\FOR{$t = 0, 1, 2, \ldots,$}
\FOR{$i = 0,  \ldots, N$}
\STATE $\text{labels}[i] = \argmin_{j \in K} \WCal(Q_j , \mu_n^i)$
\ENDFOR
\FOR{$i = 0,  \ldots, K$}
\STATE Option 1: $Q_i = FreeSupportWB(\{\mu^l:\text{labels}[l] = i \}, Y_i)$
\STATE Option 2: $Q_i = FreeSupportRPRWB(\{\mu^l:\text{labels}[l] = i \}, Y_i)$
\ENDFOR
\ENDFOR
\STATE \textbf{Output:} $\text{labels}[i], i\in [N]$.
\end{algorithmic}
\end{algorithm}

\begin{algorithm}[htb]
\caption{Free-support WB solver}
\label{alg:freeWB}
\begin{algorithmic}[1]
\STATE \textbf{Input:} $\{\mu^l = (p^l, X^l)\}_{l \in [m]}$, $Y^0$
\STATE $OBJ = WB(\{\mu^l\}_{l \in [m]}, Y^0)$
\FOR{$t = 0, 1, 2, \ldots,$}
\STATE $OBJ\_old = OBJ$;
\STATE Compute $\Pi^t = WBsolver(\{\mu^l\}_{l \in [m]}, Y^t)$, $q^t = \frac{1}{m}\Pi^t \onebf_n$, and $OBJ$;
\STATE Compute $Y^{t+1} = \frac{1}{m} X(\Pi^t)^\top \diag(1/ q)$;
\IF{$OBJ > OBJ\_old$}
\STATE break;
\ENDIF
\ENDFOR
\STATE \textbf{Output:} $\Pi^t, Y^t$.
\end{algorithmic}
 \end{algorithm}

The D2 clustering algorithm solves a fixed-support WB problem in each iteration using Algorithm \ref{alg:freeWB}, and PD2 clustering algorithm solves a fixed-support RPRWB problem in each iteration using Algorithm \ref{alg:freePRWB}. 

\begin{algorithm}[htb]
\caption{Free-support RPRWB solver}
\label{alg:freePRWB}
\begin{algorithmic}[1]
\STATE \textbf{Input:} $\{\mu^l = (p^l, X^l)\}_{l \in [m]}$, $Y^0$
\STATE $OBJ = RPRWB(\{\mu^l\}_{l \in [m]}, Y^0)$
\FOR{$t = 0, 1, 2, \ldots,$}
\STATE $OBJ\_old = OBJ$;
\STATE Compute $\Pi^t = RBCD(\{\mu^l\}_{l \in [m]}, Y^t)$, $q^t = \frac{1}{m}\Pi^t \onebf_n$, and $OBJ$;
\STATE Compute $Y^{t+1} = \frac{1}{m} X(\Pi^t)^\top \diag(1/ q)$;
\IF{$OBJ > OBJ\_old$}
\STATE break;
\ENDIF
\ENDFOR
\STATE \textbf{Output:} $\Pi^t, Y^t$.
\end{algorithmic}
\end{algorithm}

\end{document}